\documentclass[twoside]{article}
\usepackage[accepted]{aistats2026}

\usepackage{amsmath, amssymb, amsthm}
\usepackage{mathtools}
\usepackage{graphicx}
\usepackage{booktabs}
\usepackage{microtype}
\usepackage[round]{natbib}

\usepackage{enumitem}
\usepackage{url}
\usepackage{multirow}
\usepackage[ruled,vlined]{algorithm2e}
\usepackage{float}
\usepackage{hyperref}

\makeatletter
\newenvironment{algorithmH}{%
  \par\begingroup\@twocolumnfalse\begin{algorithm}[H]
}{%
  \end{algorithm}\endgroup\par
}
\makeatother

\newtheorem{theorem}{Theorem}
\newtheorem{lemma}{Lemma}
\newtheorem{corollary}{Corollary}
\theoremstyle{definition}

\theoremstyle{remark}
\newtheorem{remark}{Remark}

\newtheorem{proposition}{Proposition}

\runningtitle{Numerical Fragility in Transformers}
\begin{document}
\twocolumn[
\aistatstitle{Numerical Fragility in Transformers:\\A Layer-wise Theory for Risk Estimation and Selective Stabilization}
\aistatsauthor{Jinwoo Baek}
\aistatsaddress{Department of Computer Science, Oregon State University}
]

\begin{abstract}
Low-precision execution can induce substantial forward discrepancies in
Transformers even for fixed weights and input, yet these discrepancies are
usually monitored only at the output and lack a layer-wise theoretical account.
We develop a first-order decomposition of output mismatch into layer-local
attention, LayerNorm, and residual-transport terms, and derive from it a
practical causal risk estimator and a budgeted controller, Bound-Guided
Selective Stabilization (BGSS). Controlled sweeps verify the predicted local
sign, monotonicity, and transport structure. On GPT-2, the transport-aware
combined predictor is positively correlated with FP32-reference mismatch in all
$18$ runs and improves over a no-transport ablation in $17/18$ runs.
Reference-patch attribution shows that the same score preserves useful layer
ordering information (mean Spearman $0.362$). In budget-matched mitigation,
BGSS outperforms random same-budget control in onset events
($10.67$ vs.~$11.67$), final mismatch
($1.243\times 10^{-3}$ vs.~$1.284\times 10^{-3}$), and worst-case mismatch
($3.14\times 10^{-3}$ vs.~$8.49\times 10^{-3}$), while matching a risk-only
same-budget controller on onset suppression and sharply reducing worst-case
mismatch ($3.14\times 10^{-3}$ vs.~$5.71\times 10^{-3}$). These results support
a theory-to-algorithm account of Transformer numerical fragility in which
finite-precision risk can be analyzed, estimated, localized, and selectively
stabilized.
\end{abstract}

\vspace{-1.5em}

\section{INTRODUCTION}
\label{sec:introduction}

Modern Transformer systems \citep{Vaswani2017attention} are routinely executed
in reduced precision to increase throughput and reduce memory cost
\citep{Micikevicius2018mixedprecision,Kalamkar2019bfloat16,Dettmers2022llmint8,Yao2022zeroquant,Frantar2022gptq,Xiao2022smoothquant,Lin2024awq,Bondarenko2023quantizable}.
Yet low-precision execution can produce substantial forward discrepancies even
for fixed weights and input, and these discrepancies are often treated as a
global implementation artifact rather than a structured internal phenomenon.
Output-level mismatch reveals that a run is numerically unstable, but not which
layers are responsible, which mechanisms dominate, or where a limited
stabilization budget should be spent.

This paper develops a theory-to-algorithm account of Transformer numerical
fragility. We start from a first-order layer-wise decomposition of
finite-precision forward error into attention-side sensitivity,
LayerNorm-driven instability with explicit $\varepsilon$-dependence, and
downstream residual transport. From this decomposition, we derive a practical
causal risk estimator and BGSS, a budgeted controller that selectively
increases LayerNorm $\varepsilon$ only where predicted risk is large and
LayerNorm-dominated.

The empirical results support the full pipeline: controlled sweeps verify the
predicted local sign, monotonicity, and transport structure; on GPT-2, the
transport-aware combined predictor is positively correlated with FP32-reference
mismatch in all $18$ runs and improves over a no-transport ablation in $17/18$
runs; reference-patch attribution preserves useful layer ordering information;
and budget-matched intervention shows that BGSS improves on random same-budget
control while substantially tightening worst-case behavior relative to a
risk-only same-budget controller.

The main contribution is a unified view of numerical fragility in Transformers:
\vspace{-1em}
\begin{enumerate}[leftmargin=1.2em]
  \item We develop a first-order layer-wise theory that combines
  attention-side sensitivity, residual relaxation, and LayerNorm
  $\varepsilon$-dependence into a unified forward-fragility decomposition.

  \vspace{-0.5em}

  \item We derive a practical causal estimator and a budgeted selective controller, BGSS, from that decomposition rather than treating monitoring and mitigation as unrelated engineering heuristics.

  \vspace{-0.5em}

  \item We validate the resulting pipeline empirically through controlled local checks, GPT-2 end-to-end predictor evaluation, exact-ish attribution fidelity, and budget-matched intervention experiments.
\end{enumerate}

\vspace{-1em}

Section~\ref{sec:problem_formulation} formalizes numerical fragility at the
output and layer level. Section~\ref{sec:theory} develops the decomposition,
Section~\ref{sec:algorithm} turns it into a causal estimator and BGSS, and
Section~\ref{sec:experiments} validates the resulting pipeline.

\noindent\textbf{Code availability.} The official code release for this paper is
available at\linebreak \url{https://github.com/JinwooBaek00/Numerical-Fragility-in-Transformers}.

\vspace{-0.5em}

\section{RELATED WORK}
\label{sec:related_work}

Low-precision execution is primarily motivated by efficiency. Mixed-precision
training, BF16 execution, and recent Transformer/LLM quantization methods such
as Q8BERT, LLM.int8, ZeroQuant, GPTQ, SmoothQuant, AWQ, and Quantizable
Transformers
\citep{Micikevicius2018mixedprecision,Kalamkar2019bfloat16,Zafrir2019q8bert,Dettmers2022llmint8,Yao2022zeroquant,Frantar2022gptq,Xiao2022smoothquant,Lin2024awq,Bondarenko2023quantizable}
show that large models can often run accurately under aggressive precision
constraints, but they typically evaluate numerical behavior through task
accuracy or output-level degradation rather than a layer-wise causal account of
where fragility originates.

Our work is also connected to classical numerical stability
\citep{Goldberg1991floatingpoint,Higham2002accuracy,IEEE754-2019}, to
architectural analyses of residual and normalization structure
\citep{He2016deepresidual,Ba2016layernorm,Xiong2020onlayernorm,Zhang2019rmsnorm},
and to activation-patching or causal-tracing style analyses that localize
internal responsibility by replacing hidden activations
\citep{Vig2020cma,Meng2022factassoc}. We draw on these perspectives, but our
goal is different: we start from output mismatch and derive a unified
layer-wise risk score together with a budgeted selective stabilizer.

\vspace{-0.5em}

\section{PROBLEM FORMULATION}
\label{sec:problem_formulation}

We formulate numerical fragility in Transformers as a layer-wise causal risk
estimation problem under finite-precision execution. Our focus is not
optimization, approximation, or generalization error, but the forward
discrepancy induced purely by finite-precision arithmetic inside a fixed model.

At optimization step $t$, let $X_{L,t}$ denote the exact final hidden state of a depth-$L$
Transformer on the current minibatch, and let $\widetilde{X}_{L,t}$ denote the corresponding
finite-precision output. We define the output-level numerical mismatch by
\begin{equation}
\label{eq:pf_mismatch}
m_t
\;:=\;
\frac{\|\widetilde{X}_{L,t}-X_{L,t}\|}{\|X_{L,t}\|}.
\end{equation}
A central premise is that numerical fragility is \emph{layer-local but globally
accumulated}. Accordingly, the problem is not merely to monitor a scalar
mismatch, but to construct a
layer-wise predictor
\[
\widehat{G}_t
=
(\widehat{G}_{1,t},\dots,\widehat{G}_{L,t}),
\qquad
\widehat{R}_t
:=
\sum_{\ell=1}^{L}\widehat{G}_{\ell,t},
\]
such that $\widehat{R}_t$ tracks $m_t$ while the coordinates
$\widehat{G}_{\ell,t}$ identify which layers contribute most strongly to
fragility. This layer-wise formulation supports both localization and control:
it distinguishes attention-side from LayerNorm-driven risk and enables
selective stabilization when only a few layers can be protected. We focus on
the minimal intervention class of LayerNorm stabilizers, asking for an online
rule that identifies layers whose predicted contribution is both large and
LayerNorm-driven and stabilizes only those layers.

\section{THEORY}
\label{sec:theory}

We develop a layer-wise first-order theory for numerical fragility in
Transformers and combine attention, residual, and LayerNorm effects into a
unified forward-stability theorem.

\subsection{Setup and Notation}
\label{subsec:setup}

We work under the standard floating-point model
\[
\mathrm{fl}(a \circ b) = (a \circ b)(1+\delta),
\qquad
|\delta| \le \epsilon_{\mathrm{mach}},
\]
for $\circ \in \{+,-,\times,\div\}$, where $\epsilon_{\mathrm{mach}}$ is the
machine precision of the active compute format. For structured kernels such as
GEMMs, reductions, softmax, and LayerNorm, first-order rounding effects are
absorbed into implementation-dependent constants. Unless stated otherwise,
$\|\cdot\|$ denotes the Frobenius norm and $\|\cdot\|_2$ the spectral norm.
For an invertible linear map $T$, write $\kappa(T):=\|T\|_2\|T^{-1}\|_2$. Let
$X_\ell$ and $\widetilde{X}_\ell$ denote exact and finite-precision hidden
states, let $E_\ell:=\widetilde{X}_\ell-X_\ell$, and for a residual block
$x\mapsto x+f(x)$ define $\rho_f:=\|J_f\|_2$. The
small-gain regime $\rho_f<1$ is used only in Theorem~\ref{thm:residual} and
Corollary~\ref{cor:depth}; the unified forward-error result later needs only
$\|I+J_f\|_2 \le 1+\rho_f$.

\subsection{Self-Attention Forward Error}
\label{subsec:attn}

Let $Q,K,V \in \mathbb{R}^{n \times d}$, where $n$ is the sequence length and $d$ is the head width.
Define
\[
S = \frac{1}{\sqrt d} QK^\top \in \mathbb{R}^{n\times n},
\qquad
P = \mathrm{softmax}(S),
\]
\[
A = PV \in \mathbb{R}^{n\times d}.
\]

Let $\mathcal{S}:\mathbb{R}^{n\times n}\to\mathbb{R}^{n\times n}$ denote the
row-wise softmax map. For a probability row $p \in \mathbb{R}^n$, define
\[
J(p) = \mathrm{Diag}(p) - pp^\top,
\qquad
0 \le \|J(p)\|_2 \le \tfrac12.
\]
Since $D\mathcal{S}_S$ is block diagonal with row blocks $J(P_{i:})$, let
\[
d_{\mathrm{smx}}
=
\|D\mathcal{S}_S\|_{F\to F}
=
\max_{1 \le i \le n}\|J(P_{i:})\|_2.
\]
We then define
\[
\kappa_{\mathrm{softmax}}
:=
\frac{\|S\|}{\|P\|}\,d_{\mathrm{smx}},
\qquad
\chi_{\mathrm{score}}
:=
\frac{\|Q\|\,\|K\|}{\|P\|\,\sqrt d}\,d_{\mathrm{smx}},
\]
and, whenever $\|A\|>0$,
\[
\kappa_{\mathrm{val}}
:=
\frac{\|P\|\,\|V\|_2}{\|A\|}.
\]

\begin{theorem}[Self-Attention Forward Error]
\label{thm:attn}
Under the floating-point model above, the finite-precision result
$\widetilde{A}$ of $A = PV$ satisfies
\begin{equation}
\label{eq:attn-bound}
\begin{split}
\|\widetilde{A}-A\|
\;\le\;&
\Big[
c_{\mathrm{smx}}
+
\kappa_{\mathrm{softmax}}
+ c_{\mathrm{gemm}}\chi_{\mathrm{score}}
\\
&\quad
+
c'_{\mathrm{gemm}}
\Big]
\epsilon_{\mathrm{mach}}\,\|P\|\,\|V\|_2
+
O(\epsilon_{\mathrm{mach}}^2).
\end{split}
\end{equation}
If additionally $\|A\|>0$, then
\[
\begin{aligned}
\frac{\|\widetilde{A}-A\|}{\|A\|}
&\le
\Big[
c_{\mathrm{smx}}
+
\kappa_{\mathrm{softmax}}
+
c_{\mathrm{gemm}}\chi_{\mathrm{score}}
+
c'_{\mathrm{gemm}}
\Big]
\\
&\times\;\epsilon_{\mathrm{mach}}\,\kappa_{\mathrm{val}}
+
O(\epsilon_{\mathrm{mach}}^2).
\end{aligned}
\]
\end{theorem}

The bound separates direct softmax sensitivity, score-formation error, and
value amplification; proof appears in App.~\ref{app:attn-proof}.

\subsection{Residual Relaxation}
\label{subsec:residual}

Residual connections do not remove local floating-point error, but they can
weaken its depth-wise accumulation. Let a residual block be $x \mapsto x + f(x)$
and write $J_f$ for the Jacobian of $f$.

\begin{theorem}[Residual Stabilization]
\label{thm:residual}
If $\|J_f\|_2 < 1$, then the linearized residual map $T = I + J_f$ is invertible and satisfies
\[
\kappa(T)
=
\|I+J_f\|_2\,\|(I+J_f)^{-1}\|_2
\le
\frac{1+\|J_f\|_2}{1-\|J_f\|_2}.
\]
Hence, under the small-gain condition, residual connections relax depth-wise compounding by
keeping the local linearized map well-conditioned.
\end{theorem}

\begin{corollary}[Depth-wise relaxation]
\label{cor:depth}
For a stack of residual blocks with Jacobians $\{J_{f_\ell}\}_{\ell=1}^L$ and
$\rho_\ell := \|J_{f_\ell}\|_2 < 1$,
\[
\kappa\!\Bigl(\prod_{\ell=1}^L (I+J_{f_\ell})\Bigr)
\le
\prod_{\ell=1}^L \frac{1+\rho_\ell}{1-\rho_\ell}.
\]
In the first-order predictor developed below, we use the relaxed downstream factor
\[
\prod_{k=\ell+1}^L (1+\rho_k)
\]
to capture this attenuation effect at the level of layer-wise accumulation.
\end{corollary}

The exact condition-number bound is
$\prod_{\ell=1}^L \frac{1+\rho_\ell}{1-\rho_\ell}$, whereas the factor
$\prod_{k=\ell+1}^L (1+\rho_k)$ used later is the corresponding first-order
forward-error transport factor. Proofs appear in
App.~\ref{app:residual-proof}.

\subsection{LayerNorm Forward Error and \texorpdfstring{$\varepsilon$}{epsilon}-Regime Structure}
\label{subsec:ln}

For a feature vector $x \in \mathbb{R}^{d_{\mathrm{model}}}$, define the
$\varepsilon$-dependent LayerNorm normalization path
\[
Z^{\mathrm{LN}}_\varepsilon(x)
:=
\mathrm{Diag}(\gamma)\frac{x-\mu(x)}{\sqrt{\sigma^2(x)+\varepsilon}},
\]
so that the full LayerNorm output is
\[
\mathrm{LN}(x)=Z^{\mathrm{LN}}_\varepsilon(x)+\beta.
\]
Here $\mu(x)$ and $\sigma^2(x)$ denote the mean and variance over the normalized
axis, and $\varepsilon > 0$ is the stabilizer. Because BGSS acts only through
$\varepsilon$, we isolate the normalization path
$Z^{\mathrm{LN}}_\varepsilon$ and absorb the final bias shift into the
$\varepsilon$-independent remainder term of the unified theorem. We also track
the scale-aware indicator
$\rho_{\mathrm{LN}}(\varepsilon):=(\sigma^2(x)/\varepsilon)\,
d_{\mathrm{model}}\,\epsilon_{\mathrm{mach}}$, which serves only as a practical
proxy for entry into the $\varepsilon$-dominated regime.

\begin{proposition}[First-order LayerNorm normalization-path forward error]
\label{prop:ln}
Under the floating-point model, there exists a
kernel- and dimension-dependent constant $a_{\mathrm{ln}}>0$ such that
\begin{equation}
\label{eq:ln-abs-bound}
\|\widetilde{Z}^{\mathrm{LN}}_\varepsilon(x)-Z^{\mathrm{LN}}_\varepsilon(x)\|
\le
\epsilon_{\mathrm{mach}}\,M_{\mathrm{LN}}(x,\varepsilon)
+ O(\epsilon_{\mathrm{mach}}^2),
\end{equation}
where
\begin{equation}
\label{eq:ln-abs-mag}
M_{\mathrm{LN}}(x,\varepsilon)
:=
\|\mathrm{Diag}(\gamma)\|_2
\frac{\varepsilon+2\sigma^2(x)}{(\sigma^2(x)+\varepsilon)^{3/2}}
\bigl(a_{\mathrm{ln}}\|x\|_2\bigr).
\end{equation}
If additionally $\|Z^{\mathrm{LN}}_\varepsilon(x)\|>0$, then
\begin{equation}
\label{eq:ln-bound}
\frac{\|\widetilde{Z}^{\mathrm{LN}}_\varepsilon(x)-Z^{\mathrm{LN}}_\varepsilon(x)\|}{\|Z^{\mathrm{LN}}_\varepsilon(x)\|}
\le
\epsilon_{\mathrm{mach}}\,C_{\mathrm{LN}}(x,\varepsilon)
+ O(\epsilon_{\mathrm{mach}}^2),
\end{equation}
where
\begin{equation}
\label{eq:ln-coeff}
C_{\mathrm{LN}}(x,\varepsilon)
:=
\frac{M_{\mathrm{LN}}(x,\varepsilon)}{\|Z^{\mathrm{LN}}_\varepsilon(x)\|}.
\end{equation}
Moreover, for fixed $x$, $\gamma$, and $a_{\mathrm{ln}}$, the
$\varepsilon$-dependent factor
\[
f_{\sigma^2(x)}(\varepsilon)
:=
\frac{\varepsilon+2\sigma^2(x)}{(\sigma^2(x)+\varepsilon)^{3/2}}
\]
is strictly decreasing in $\varepsilon>0$.
\end{proposition}

Proposition~\ref{prop:ln} gives the first-order $\varepsilon$-dependent
normalization-path magnitude together with a relative coefficient when
$\|Z^{\mathrm{LN}}_\varepsilon(x)\|>0$; proof appears in App.~\ref{app:ln-proof}.

\subsection{Unified Forward Stability}
\label{subsec:unified}

Let $A_\ell$ denote the exact attention-side coefficient
\begin{equation}
\label{eq:attn-contrib}
\begin{split}
A_\ell
:=\;&
\Bigl[
c_{\mathrm{smx}}
+
\kappa_{\mathrm{softmax},\ell}
+
c_{\mathrm{gemm}}\chi_{\mathrm{score},\ell}
+
c'_{\mathrm{gemm}}
\Bigr]
\\
&\times\;\|P_\ell\|\,\|V_\ell\|_2,
\end{split}
\end{equation}
let $M^{\mathrm{LN}}_\ell := M_{\mathrm{LN}}(x_\ell,\varepsilon_\ell)$ be the
exact LayerNorm normalization-path magnitude from Proposition~\ref{prop:ln}.
Let $\mathcal{R}_\ell$ be the finite set of remaining non-attention kernels in
layer $\ell$ outside the $\varepsilon$-dependent normalization path, and for
each $r \in \mathcal{R}_\ell$ assume
\begin{equation}
\label{eq:eff-module}
\|\Delta_{\ell,r}\|
\le
\epsilon_{\mathrm{mach}}\,b_{\ell,r}
+ O(\epsilon_{\mathrm{mach}}^2).
\end{equation}
Let $M^{\mathrm{eff}}_\ell := \sum_{r\in\mathcal{R}_\ell} b_{\ell,r}$ denote the
aggregate remainder magnitude, which upper-bounds the summed remainder
perturbation by Lemma~\ref{lem:eff-closure}. The local fragility magnitude is
\begin{equation}
\label{eq:local-coeff}
M_\ell
:=
M^{\mathrm{eff}}_\ell
+
A_\ell
+
M^{\mathrm{LN}}_\ell.
\end{equation}

\begin{theorem}[Unified forward stability]
\label{thm:unified}
Assume the absolute bound of Theorem~\ref{thm:attn} for every attention site and
Proposition~\ref{prop:ln} for every LayerNorm site contributing to $M_\ell$.
For each residual block $x\mapsto x+f_\ell(x)$, let
\[
\rho_\ell := \|J_{f_\ell}(X_{\ell-1})\|_2.
\]
Assume further that \eqref{eq:eff-module} holds for every $r\in\mathcal{R}_\ell$
and every layer $\ell$, and that $\|X_L\|>0$.
\begin{equation}
\label{eq:unified}
\begin{split}
\frac{\|\widetilde{X}_L - X_L\|}{\|X_L\|}
\;\le\;&
\epsilon_{\mathrm{mach}}
\sum_{\ell=1}^{L}
\frac{M_\ell}{\|X_L\|}
\prod_{k=\ell+1}^{L}(1+\rho_k)
\\
&+
O(\epsilon_{\mathrm{mach}}^2).
\end{split}
\end{equation}
\end{theorem}

Theorem~\ref{thm:unified} decomposes output-level mismatch into layer-wise local
magnitudes transported by downstream residual factors, without any small-gain
assumption since $\|I+J_{f_k}\|_2\le 1+\rho_k$. Proof appears in
App.~\ref{app:unified-proof}. Any benign inter-layer norm ratios needed to
express site-local absolute magnitudes relative to the final output scale are
absorbed into the layer-dependent first-order constants defining $M_\ell$.

\subsection{From the Unified Bound to a Selective Controller}
\label{subsec:bgss_theory}

Theorem~\ref{thm:unified} naturally induces a monitored layer-wise risk score.
At step $t$ with $\|X_{L,t}\|>0$, let
\begin{equation}
\label{eq:bgss-attn}
\begin{split}
A_{\ell,t}
:=\;&
\Bigl[
c_{\mathrm{smx}}
+
\kappa_{\mathrm{softmax},\ell,t}
+
c_{\mathrm{gemm}}\chi_{\mathrm{score},\ell,t}
+
c'_{\mathrm{gemm}}
\Bigr]
\\
&\times\;\|P_{\ell,t}\|\,\|V_{\ell,t}\|_2,
\end{split}
\end{equation}
let $M^{\mathrm{LN}}_{\ell,t}:=M_{\mathrm{LN}}(x_{\ell,t},\varepsilon_{\ell,t})$
be the exact LayerNorm normalization-path magnitude at site $(\ell,t)$ from
Proposition~\ref{prop:ln}, and let
$M^{\mathrm{eff}}_{\ell,t}:=\sum_{r\in\mathcal{R}_\ell} b_{\ell,t,r}$ be the
aggregate monitored-step remainder magnitude for the non-attention,
non-normalization-path kernels in $\mathcal{R}_\ell$. Set
\begin{equation}
\label{eq:bgss-local}
M_{\ell,t}
:=
M^{\mathrm{eff}}_{\ell,t}
+
A_{\ell,t}
+
M^{\mathrm{LN}}_{\ell,t},
\end{equation}
where $\rho_{k,t} := \|J_{f_k}(X_{k-1,t})\|_2$ denotes the downstream residual
Jacobian norm at step $t$. Define
\begin{equation}
\label{eq:bgss-layer}
G_{\ell,t}
:=
\frac{M_{\ell,t}}{\|X_{L,t}\|}
\prod_{k=\ell+1}^{L}(1+\rho_{k,t}).
\end{equation}
Then Theorem~\ref{thm:unified} gives the first-order predictor
\begin{equation}
\label{eq:bgss-predictor}
\frac{\|\widetilde{X}_{L,t}-X_{L,t}\|}{\|X_{L,t}\|}
\;\lesssim\;
\epsilon_{\mathrm{mach}}
\sum_{\ell=1}^{L} G_{\ell,t}.
\end{equation}

Here $G_{\ell,t}$ is the predicted layer-wise contribution to end-to-end
fragility. To quantify whether this contribution is primarily LayerNorm-driven,
define
\begin{equation}
\label{eq:bgss-phi}
\phi_{\ell,t}
:=
\frac{M^{\mathrm{LN}}_{\ell,t}}{M_{\ell,t}},
\end{equation}
with the convention $\phi_{\ell,t}:=0$ when $M_{\ell,t}=0$. Large $G_{\ell,t}$
indicates high risk, while large $\phi_{\ell,t}$ indicates LayerNorm dominance.

\begin{proposition}[Monotone reduction of the frozen-scale LayerNorm contribution]
\label{prop:bgss}
Fix a monitored layer $\ell$ and step $t$. Let $v_{\ell,t}:=\sigma^2(x_{\ell,t})$ and let
$a_{\ell,t}>0$ denote the positive constant inherited from
Proposition~\ref{prop:ln} at the monitored site $(\ell,t)$. Hold $x_{\ell,t}$,
$v_{\ell,t}$, $a_{\ell,t}$, and $\gamma_\ell$,
$M^{\mathrm{eff}}_{\ell,t}$, $A_{\ell,t}$, $\|X_{L,t}\|$, and all
non-normalization-path coefficients fixed. Define
\begin{equation}
\label{eq:bgss-cln}
M_{\mathrm{LN},\ell,t}^{\mathrm{frz}}(\varepsilon)
:=
\|\mathrm{Diag}(\gamma_\ell)\|_2
\frac{\varepsilon+2v_{\ell,t}}{(v_{\ell,t}+\varepsilon)^{3/2}}
\bigl(a_{\ell,t}\,\|x_{\ell,t}\|_2\bigr).
\end{equation}
Then $M_{\mathrm{LN},\ell,t}^{\mathrm{frz}}$ is nonincreasing in $\varepsilon>0$, and it is
strictly decreasing whenever
$\|\mathrm{Diag}(\gamma_\ell)\|_2\,\|x_{\ell,t}\|_2>0$. Moreover,
\begin{equation}
\label{eq:bgss-derivative}
\begin{split}
\frac{d}{d\varepsilon}M_{\mathrm{LN},\ell,t}^{\mathrm{frz}}(\varepsilon)
:=\;&
-\frac{\|\mathrm{Diag}(\gamma_\ell)\|_2}{2}
\bigl(a_{\ell,t}\,\|x_{\ell,t}\|_2\bigr)
\\
&\times\;\frac{\varepsilon+4v_{\ell,t}}{(v_{\ell,t}+\varepsilon)^{5/2}}
\\
\le\;&0.
\end{split}
\end{equation}
Consequently, the frozen-scale layer-wise contribution
\begin{equation}
\label{eq:bgss-g-eps}
G_{\ell,t}^{\mathrm{frz}}(\varepsilon)
:=
\frac{
M^{\mathrm{eff}}_{\ell,t}
+
A_{\ell,t}
+
M_{\mathrm{LN},\ell,t}^{\mathrm{frz}}(\varepsilon)
}{\|X_{L,t}\|}
\prod_{k=\ell+1}^{L}(1+\rho_{k,t})
\end{equation}
is nonincreasing in $\varepsilon$, and for any $\varepsilon^{\prime}\ge\varepsilon$
\begin{equation}
\label{eq:bgss-reduction}
\begin{split}
G_{\ell,t}^{\mathrm{frz}}(\varepsilon)
- G_{\ell,t}^{\mathrm{frz}}(\varepsilon^{\prime})
=\;&
\frac{1}{\|X_{L,t}\|}
\\
&\times\;\Bigl(
M_{\mathrm{LN},\ell,t}^{\mathrm{frz}}(\varepsilon)
-
M_{\mathrm{LN},\ell,t}^{\mathrm{frz}}(\varepsilon^{\prime})
\Bigr)
\\
&\times\;\prod_{k=\ell+1}^{L}(1+\rho_{k,t}).
\end{split}
\end{equation}
\end{proposition}
Proposition~\ref{prop:bgss} shows that, in the frozen-scale first-order model,
increasing $\varepsilon$ decreases the LayerNorm part of the predicted
layer-wise contribution, motivating selective intervention on layers with both
large $G_{\ell,t}$ and large $\phi_{\ell,t}$. In the algorithmic
implementation, common positive kernel-dependent multipliers are absorbed into
controller thresholds. Proof appears in App.~\ref{app:bgss-proof}.

\section{ALGORITHM}
\label{sec:algorithm}

We instantiate the theory of Section~\ref{sec:theory} as
\emph{Bound-Guided Selective Stabilization} (BGSS), an online controller that
monitors layer-wise numerical fragility and selectively increases LayerNorm
stabilizers. BGSS leaves model weights, optimizer states, and architectural
blocks unchanged; it acts only through LayerNorm $\varepsilon$ values.
Throughout this section, hats denote quantities measured on the monitored
finite-precision forward pass at step $t$. BGSS is theory-guided rather than
exactly calibrated: the practical statistics below need not equal the exact
coefficients from Section~\ref{subsec:bgss_theory}, but they are required to be
nonnegative, causal, and to preserve the exact score structure
\[
\begin{aligned}
&\text{local magnitude}
\\
&\;+\;\text{downstream residual transport}
\\
&\;+\;\text{monotone }\varepsilon\text{-dependence}
\\
&\qquad\text{of the LayerNorm term.}
\end{aligned}
\]
Unknown hardware/runtime-dependent scale factors are therefore absorbed into the
controller thresholds.

\subsection{Practical Layer-Wise Risk Estimation}
\label{subsec:bgss_estimation}

At a monitored step $t$ with $\|\widehat{X}_{L,t}\|>0$, the exact coefficients in
Section~\ref{subsec:bgss_theory} are either unavailable or unnecessarily costly
to compute online. BGSS therefore replaces them with causal surrogates computed
from the monitored pass. For each layer $\ell$, let $\widehat{S}_{\ell,t}$,
$\widehat{Q}_{\ell,t}$, $\widehat{K}_{\ell,t}$, $\widehat{P}_{\ell,t}$,
$\widehat{V}_{\ell,t}$, and $\widehat{Z}^{\mathrm{LN}}_{\ell,t}$ denote the
observed score, query, key, attention-probability, value, and LayerNorm
normalization-path tensors. Let
$\widehat{M}^{\mathrm{eff}}_{\ell,t}\ge 0$ denote any causal surrogate of the
remaining non-attention, non-normalization-path local magnitude, and let
$\widehat{\rho}_{k,t}\ge 0$ denote any causal surrogate of the downstream
residual Jacobian norm $\rho_{k,t}$.

We first estimate the local softmax differential by
\[
\widehat{d}_{\mathrm{smx},\ell,t}
:=
\widehat{\|D\mathcal{S}\|}_{\ell,t},
\]
\[
\widehat{\kappa}_{\mathrm{softmax},\ell,t}
:=
\frac{\|\widehat{S}_{\ell,t}\|}{\|\widehat{P}_{\ell,t}\|}\,
\widehat{d}_{\mathrm{smx},\ell,t},
\]
\[
\widehat{\chi}_{\mathrm{score},\ell,t}
:=
\frac{\|\widehat{Q}_{\ell,t}\|\,\|\widehat{K}_{\ell,t}\|}
{\|\widehat{P}_{\ell,t}\|\,\sqrt d}\,
\widehat{d}_{\mathrm{smx},\ell,t},
\]
\[
\widehat{A}_{\ell,t}
:=
\bigl(
\widehat{\kappa}_{\mathrm{softmax},\ell,t}
+
\widehat{\chi}_{\mathrm{score},\ell,t}
\bigr)
\|\widehat{P}_{\ell,t}\|\,\|\widehat{V}_{\ell,t}\|_2,
\]
and use the monotone LayerNorm surrogate
\[
\widehat{C}_{\mathrm{LN},\ell,t}
:=
\frac{\varepsilon_{\ell,t}+2\widehat{\sigma}_{\ell,t}^{\,2}}
{(\widehat{\sigma}_{\ell,t}^{\,2}+\varepsilon_{\ell,t})^{3/2}}.
\]
This preserves the exact monotone dependence on $\varepsilon$ while omitting shared
positive kernel-dependent multipliers. Define the surrogate LayerNorm magnitude
\[
\widehat{M}^{\mathrm{LN}}_{\ell,t}
:=
\|\widehat{Z}^{\mathrm{LN}}_{\ell,t}\|\,\widehat{C}_{\mathrm{LN},\ell,t},
\]
\[
\widehat{M}_{\ell,t}
:=
\widehat{M}^{\mathrm{eff}}_{\ell,t}
+
\widehat{A}_{\ell,t}
+
\widehat{M}^{\mathrm{LN}}_{\ell,t},
\]
the practical layer-wise risk score
\begin{equation}
\label{eq:alg-risk}
\widehat{G}_{\ell,t}
:=
\frac{\widehat{M}_{\ell,t}}{\|\widehat{X}_{L,t}\|}
\prod_{k=\ell+1}^{L}(1+\widehat{\rho}_{k,t}),
\end{equation}
and the LayerNorm dominance ratio
\begin{equation}
\label{eq:alg-phi}
\widehat{\phi}_{\ell,t}
:=
\frac{\widehat{M}^{\mathrm{LN}}_{\ell,t}}{\widehat{M}_{\ell,t}}.
\end{equation}
We adopt the convention $\widehat{\phi}_{\ell,t}:=0$ when
$\widehat{M}_{\ell,t}=0$.

To detect entry into the $\varepsilon$-dominated regime, BGSS also tracks
\begin{equation}
\label{eq:alg-rho-ln}
\widehat{\rho}_{\mathrm{LN},\ell,t}
:=
\frac{\widehat{\sigma}_{\ell,t}^{\,2}}{\varepsilon_{\ell,t}}\,
d_{\mathrm{model}}\,
\epsilon_{\mathrm{mach}},
\end{equation}
so that $\widehat{G}_{\ell,t}$ preserves the exact
\[
\begin{aligned}
&\text{local magnitude}\;/\;\text{final-output scale}
\\
&\times\;\text{downstream transport}
\end{aligned}
\]
structure, $\widehat{\phi}_{\ell,t}$ measures how much of that surrogate risk is
LayerNorm-driven, and $\widehat{\rho}_{\mathrm{LN},\ell,t}$ tests whether the layer is in the
$\varepsilon$-dominated regime. We do not claim exact calibration of
$\widehat{G}_{\ell,t}$ across kernels or runtimes; BGSS uses these quantities
for ranking, thresholding, and local action selection. Estimation details are
given in App.~\ref{app:estimation}.

\paragraph{Selection rule.}
BGSS forms the eligible set
\[
\mathcal{E}_t
:=
\left\{
\ell:
\begin{array}{l}
\widehat{G}_{\ell,t}\ge\tau_G,\\
\widehat{\phi}_{\ell,t}\ge\tau_\phi,\\
\widehat{\rho}_{\mathrm{LN},\ell,t}<1
\end{array}
\right\},
\]
where $\tau_G$ is a risk threshold and $\tau_\phi$ is a LayerNorm-dominance
threshold. A layer is therefore eligible only when its predicted contribution is
large, LayerNorm-dominant, and currently in the $\varepsilon$-dominated regime.
Layers under active cooldown are removed from $\mathcal{E}_t$, and BGSS keeps
the top-$B$ remaining layers ranked by $\widehat{G}_{\ell,t}$.

\paragraph{Update rule.}
For each selected layer $\ell$, BGSS applies the monotone update
\begin{equation}
\label{eq:alg-eps-cand}
\varepsilon^{\mathrm{cand}}_{\ell,t}
:=
\mathrm{clip}\!\left(
\frac{\widehat{\sigma}_{\ell,t}^{\,2}\,d_{\mathrm{model}}\,\epsilon_{\mathrm{mach}}}{\rho_*},
\ \varepsilon_{\min},\ \varepsilon_{\max}
\right),
\end{equation}
\begin{equation}
\label{eq:alg-eps-update}
\varepsilon_{\ell,t+1}
:=
\max\!\bigl\{
\varepsilon_{\ell,t},\,
\varepsilon^{\mathrm{cand}}_{\ell,t}
\bigr\},
\end{equation}
where $\rho_* \in (0,1)$ is the target post-update LayerNorm ratio.
All non-selected layers keep
\[
\varepsilon_{\ell,t+1} = \varepsilon_{\ell,t}.
\]
The candidate in \eqref{eq:alg-eps-cand} is the smallest value inside
$[\varepsilon_{\min},\varepsilon_{\max}]$ that would enforce
$\widehat{\rho}_{\mathrm{LN},\ell,t}\le \rho_*$ on the monitored statistics
whenever feasible; the outer max in \eqref{eq:alg-eps-update} is then the
smallest monotone bounded update. If the target value exceeds
$\varepsilon_{\max}$, the rule saturates at $\varepsilon_{\max}$ and therefore
achieves the best feasible reduction of the monitored ratio subject to the box
constraint.

Under frozen monitored statistics, the post-update monitored ratio satisfies
\[
\begin{aligned}
\frac{\widehat{\sigma}_{\ell,t}^{\,2}\,d_{\mathrm{model}}\,\epsilon_{\mathrm{mach}}}
{\varepsilon_{\ell,t+1}}
&\le
\rho_*
\\
\text{whenever }\quad
\frac{\widehat{\sigma}_{\ell,t}^{\,2}\,d_{\mathrm{model}}\,\epsilon_{\mathrm{mach}}}{\rho_*}
&\le
\varepsilon_{\max},
\end{aligned}
\]
and the frozen surrogate LayerNorm magnitude
\[
\widehat{M}^{\mathrm{LN,frz}}_{\ell,t}(\varepsilon)
:=
\|\widehat{Z}^{\mathrm{LN}}_{\ell,t}\|
\frac{\varepsilon+2\widehat{\sigma}_{\ell,t}^{\,2}}
{(\widehat{\sigma}_{\ell,t}^{\,2}+\varepsilon)^{3/2}}
\]
is nonincreasing in $\varepsilon$. Consequently, the corresponding frozen
surrogate layer score is also nonincreasing. This is the algorithmic analogue of
Proposition~\ref{prop:bgss}.

\begin{algorithm}[t]
\caption{Bound-Guided Selective Stabilization (BGSS)}
\label{alg:bgss}
\footnotesize
\setlength{\algomargin}{0.8em}
\DontPrintSemicolon
\KwIn{monitor interval $m$; thresholds $(\tau_G,\tau_\phi)$; target ratio $\rho_*$; budget $B$; cooldown $c$; bounds $(\varepsilon_{\min},\varepsilon_{\max})$}
Initialize cooldown counters to zero for all layers\;
\For{optimization step $t=1,2,\dots$}{
  perform the usual forward/backward/update step\;
  \If{$t \bmod m = 0$}{
    decrement all positive cooldown counters by one\;
    estimate $\widehat{G}_{\ell,t}$, $\widehat{\phi}_{\ell,t}$, and $\widehat{\rho}_{\mathrm{LN},\ell,t}$ for all layers $\ell$ using the activations observed at step $t$\;
    form $\mathcal{E}_t=\{\ell:\widehat{G}_{\ell,t}\ge\tau_G,\ \widehat{\phi}_{\ell,t}\ge\tau_\phi,\ \widehat{\rho}_{\mathrm{LN},\ell,t}<1\}$\;
    remove layers with positive cooldown counters from $\mathcal{E}_t$\;
    keep the top-$B$ layers in $\mathcal{E}_t$ ranked by $\widehat{G}_{\ell,t}$\;
    \ForEach{selected layer $\ell$}{
      $\varepsilon^{\mathrm{cand}}_{\ell,t}
      \gets
      \mathrm{clip}\!\left(
      \widehat{\sigma}_{\ell,t}^{\,2} d_{\mathrm{model}} \epsilon_{\mathrm{mach}}/\rho_*,
      \varepsilon_{\min},\varepsilon_{\max}
      \right)$\;
      $\varepsilon^{\mathrm{new}}_{\ell,t+1}
      \gets
      \max\{\varepsilon_{\ell,t},\varepsilon^{\mathrm{cand}}_{\ell,t}\}$\;
      \If{$\varepsilon^{\mathrm{new}}_{\ell,t+1}>\varepsilon_{\ell,t}$}{
        set $\varepsilon_{\ell,t+1}\gets\varepsilon^{\mathrm{new}}_{\ell,t+1}$\;
        set the cooldown counter of layer $\ell$ to $c$\;
      }
    }
  }
}
\end{algorithm}

\paragraph{Properties and theoretical grounding.}
BGSS is causal, selective, monotone, and mechanism-aware. Causality follows from
using only statistics observed at step $t$ and applying updates from step $t+1$
onward. Selectivity follows from the top-$B$ budget on eligible layers. Monotonicity
follows from \eqref{eq:alg-eps-update}. Mechanism-awareness follows from requiring
both large overall risk ($\widehat{G}_{\ell,t}$) and large LayerNorm dominance
($\widehat{\phi}_{\ell,t}$). Finally, Theorem~\ref{thm:unified} provides the
layer-wise decomposition underlying $\widehat{G}_{\ell,t}$, while
Proposition~\ref{prop:bgss} and the frozen-statistics argument above justify
\eqref{eq:alg-eps-update} as the smallest bounded monotone local
risk-reducing action toward the target ratio $\rho_*$. Thus, BGSS is a causal
surrogate-based algorithmic realization of the unified first-order fragility
model under budgeted intervention constraints.

\section{EXPERIMENTS}
\label{sec:experiments}

\paragraph{Evaluation protocol.}
Our empirical evaluation mirrors the theory-to-algorithm structure of the
paper. Throughout, FP32 execution serves as the numerical reference. E1 uses
synthetic controlled sweeps to test the local statements behind
Theorem~\ref{thm:attn}, Theorem~\ref{thm:residual},
Corollary~\ref{cor:depth}, and Proposition~\ref{prop:ln}. E2 and E3 use the
HuggingFace \texttt{gpt2} checkpoint with BF16/FP16 monitored passes on
WikiText-103 validation, sequence lengths $\{128,512,1024\}$, $3$ seeds, and
$96$ monitored windows per run, for $18$ completed runs. E3 reuses these E2
runs and evaluates layer attribution on the top-$8$ highest-mismatch windows
per run. E5 uses stable FP32-master / FP16-shadow training on WikiText-2 train
with sequence length $256$, $256$ monitored steps, $3$ seeds, and a shared
budget of $24$ actions for the budget-matched controllers. Unless otherwise
noted, aggregate statistics are means over completed runs.
Table~\ref{tab:gpt2_summary} collects the main GPT-2 evidence for the
end-to-end predictor, attribution, and mitigation experiments in one place.
Where variability is relevant, we report standard deviations across runs or
across seeds.

\begin{table}[t]
\centering
\caption{Main GPT-2 evidence summary. Means are over completed runs; when
reported, uncertainties are standard deviations across runs/seeds.}
\label{tab:gpt2_summary}
\scriptsize
\setlength{\tabcolsep}{3pt}
\renewcommand{\arraystretch}{0.98}
\begin{tabular}{@{}p{0.10\columnwidth}p{0.84\columnwidth}@{}}
\toprule
Exp. & Main evidence \\
\midrule
E2 & The transport-aware predictor achieves Pearson $0.370\pm0.119$ versus
\texttt{no\_transport} $0.206$, improves correlation in $17/18$ runs, and
improves top-$k$ retrieval in $12/18$ runs. \\
E3 & Reference-patch attribution yields mean Spearman
$0.362\pm0.156$, pairwise accuracy $0.643\pm0.061$, top-$3$ overlap
$0.505$, and top-$5$ overlap $0.622$. \\
E5 & Relative to the random same-budget controller, BGSS reduces onset events
($10.67\pm1.15$ vs.~$11.67$), final mismatch
($1.243\times 10^{-3}\pm 2.87\times 10^{-5}$ vs.~$1.284\times 10^{-3}$),
and worst-case mismatch
($3.14\times 10^{-3}\pm 1.01\times 10^{-3}$ vs.~$8.49\times 10^{-3}$); versus
the risk-only same-budget controller, BGSS matches mean onset events
($10.67$ vs.~$10.67$) while reducing worst-case mismatch
($3.14\times 10^{-3}$ vs.~$5.71\times 10^{-3}$), at a slight cost in mean
final mismatch ($1.243\times 10^{-3}$ vs.~$1.221\times 10^{-3}$). \\
\bottomrule
\end{tabular}
\end{table}

\subsection{E1: Controlled local validation}
\label{subsec:e1}

We begin with controlled numerical checks that isolate the three local
mechanisms in the theory. The attention sweep varies score margin and value
scale in a $2\times 2$ toy attention map; the LayerNorm sweep varies
$\varepsilon$ over $10$ logarithmically spaced values; and the residual sweep
varies the local gain $\rho$ over $8$ values while composing depth-$6$
transport. These are not language-model runs; they are direct mechanism probes.

Figure~\ref{fig:e1_attention} shows that the attention proxy tracks the
measured attention-output perturbation almost perfectly
(Pearson $=0.999999$, Spearman $=1.0$), as predicted by
Theorem~\ref{thm:attn}. In the LayerNorm sweep, both the measured
normalization-path change and the causal proxy decrease monotonically with
$\varepsilon$, matching Proposition~\ref{prop:ln}. In the residual sweep, the
measured downstream amplification remains below the predicted transport bound at
all tested $\rho$ values, matching Theorem~\ref{thm:residual} and
Corollary~\ref{cor:depth}. The corresponding LayerNorm and residual curves are
reported in Appendix~\ref{app:exp-details}. Overall, E1 verifies that the
local signs, monotonicity, and transport directionality assumed by the unified
decomposition are directly visible in controlled numerical measurements.

\begin{figure}[H]
\centering
\includegraphics[width=\linewidth]{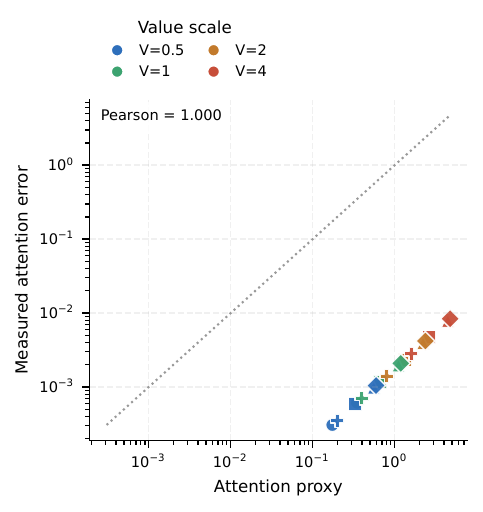}
\caption{E1 attention-side controlled validation. Each point is one controlled
attention configuration from the margin/value-scale sweep. The theory proxy and
the measured attention-output perturbation align almost perfectly, confirming
the sign and relative scaling predicted by Theorem~\ref{thm:attn}.}
\label{fig:e1_attention}
\end{figure}

\subsection{E2: End-to-end predictor validation on GPT-2}
\label{subsec:e2}

We next test whether the practical transport-aware predictor tracks the final
FP32-reference mismatch on real GPT-2 evaluation windows. For each monitored
window, we compute the combined predictor
$\widehat{R}_t=\sum_{\ell}\widehat{G}_{\ell,t}$ and compare it to the
\texttt{no\_transport} ablation, which removes the downstream residual
transport factors from the same local decomposition. This is the primary E2
comparison: the unified theory specifically adds transport on top of local
magnitudes, whereas the single-mechanism signals are diagnostic probes rather
than the main baseline.

Across the $18$ GPT-2 runs, the combined predictor is positively correlated
with the final mismatch in all runs. Averaged over the full sweep, it achieves
mean Pearson/Spearman correlations of $0.370\pm0.119$ and $0.351\pm0.060$, compared with
$0.206/0.212$ for \texttt{no\_transport}. Adding transport improves correlation
in $17/18$ runs and improves top-$k$ retrieval of high-mismatch windows in
$12/18$ runs. Figure~\ref{fig:e2_transport} summarizes this comparison: almost
all runs lie above the diagonal, showing that causal downstream transport is
consistently useful in practice. This is the main empirical validation of the
transport-aware unified estimator from Section~\ref{sec:algorithm}. A run-wise
$\Delta$-Pearson view and a binned risk-trend plot appear in
Appendix~\ref{app:exp-details}.

\begin{figure}[H]
\centering
\includegraphics[width=\linewidth]{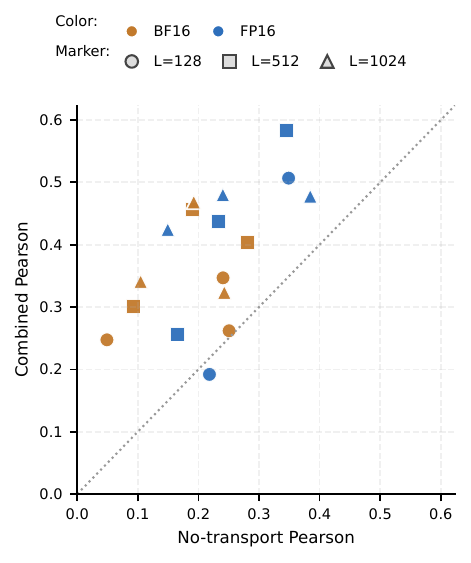}
\caption{E2 end-to-end predictor validation on GPT-2. Each point is one run
over a precision/sequence-length/seed combination. The transport-aware
combined predictor almost always improves on the no-transport ablation,
supporting the role of downstream residual transport in the unified risk
estimator.}
\label{fig:e2_transport}
\end{figure}

\subsection{E3: Attribution and localization fidelity}
\label{subsec:e3}

E2 validates the scalar predictor. E3 asks whether the induced layerwise
ranking is faithful to reference-patch layer importance. Starting from the E2
runs, we select the top-$8$ highest-mismatch windows from each run and rerun
the same manual GPT-2 forward path while replacing exactly one low-precision
block with its FP32 reference counterpart. The resulting reduction in final
mismatch is the reference-patch effect for that layer.

The first check is alignment: the recomputed baseline mismatch agrees with the
source E2 mismatch to within an average absolute gap of only
$9.24\times 10^{-7}\pm 3.47\times 10^{-7}$, so the attribution comparison is not confounded by a
different forward path. The second check is ranking fidelity. Across the
$18$ runs, we obtain mean Spearman correlation $0.362\pm0.156$, mean pairwise ordering
accuracy $0.643\pm0.061$, mean top-$3$ overlap $0.505\pm0.108$, and mean top-$5$ overlap
$0.622\pm0.084$. Figure~\ref{fig:e3_heatmap} shows the aggregate proxy-rank versus
exact-rank heatmap. The mass is concentrated near the diagonal, indicating that
the practical layerwise proxy is not a perfect oracle but does preserve useful
localization structure for the most influential layers. Run-level fidelity
statistics are reported in Appendix~\ref{app:exp-details}.

\begin{figure}[H]
\centering
\includegraphics[width=\linewidth]{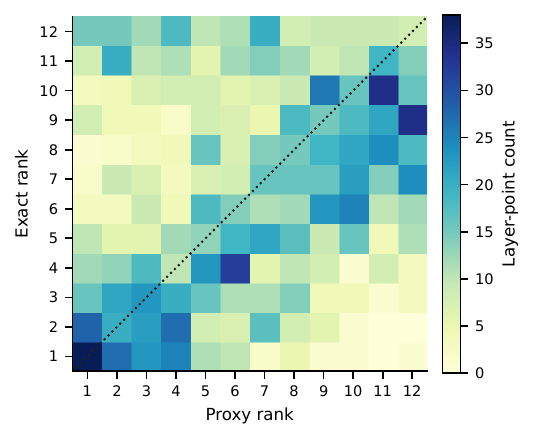}
\caption{E3 attribution fidelity. The heatmap counts proxy-rank versus exact-ish
rank over all evaluated layer-step pairs. Mass concentrated near the diagonal
indicates that the practical layerwise proxy preserves meaningful exact layer
ordering information.}
\label{fig:e3_heatmap}
\end{figure}

\subsection{E5: Budgeted mitigation utility}
\label{subsec:e5}

Finally, we test whether the theory-guided score is useful for intervention
rather than explanation alone. We compare five policies during stable
FP32-master / FP16-shadow training: no intervention, a static global
LayerNorm-$\varepsilon$ increase, a random same-budget controller, a risk-only
same-budget controller that ranks layers by overall risk without the
LayerNorm-dominance gate, and BGSS. The key comparisons are budget-matched:
BGSS, the random controller, and the risk-only controller all use $24$ actions
and the same average protected layer-steps ($2742$), so performance gaps
reflect policy quality rather than intervention volume.

Figure~\ref{fig:e5_robustness} shows the main robustness view. Relative to the
random same-budget controller, BGSS reduces mean mismatch-onset events from
$11.67$ to $10.67\pm1.15$, lowers mean final mismatch from
$1.284\times 10^{-3}$ to $1.243\times 10^{-3}\pm2.87\times 10^{-5}$, and sharply
reduces mean max mismatch from $8.49\times 10^{-3}$ to
$3.14\times 10^{-3}\pm1.01\times 10^{-3}$. The stronger mechanism-aware test is
against the risk-only same-budget controller: both policies achieve the same
mean onset-event count ($10.67$), and the risk-only controller attains a
slightly lower mean final mismatch ($1.221\times 10^{-3}$ versus
$1.243\times 10^{-3}$), but BGSS reduces mean max mismatch from
$5.71\times 10^{-3}$ to $3.14\times 10^{-3}$ and wins the worst-case comparison
in all three seeds. Static global stabilization uses a larger protection budget
($3072$ protected layer-steps) yet still produces higher final mismatch and
higher worst-case mismatch than BGSS. The conclusion is therefore not that
BGSS minimizes every average metric, but that the theory-guided policy is the
strongest \emph{robust} budget-matched controller among the tested
alternatives. A complementary tradeoff view is given in
Appendix~\ref{app:exp-details}.

\begin{figure}[H]
\centering
\includegraphics[width=\linewidth]{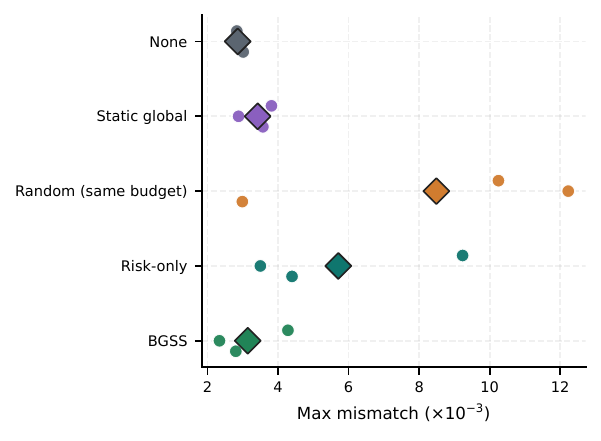}
\caption{E5 budgeted mitigation, viewed through worst-case robustness. Small
points show individual seeds and diamonds show policy means. Under the same
budget as the random and risk-only controllers, BGSS substantially reduces max
mismatch and yields the strongest robust budget-matched policy.}
\label{fig:e5_robustness}
\end{figure}

\paragraph{Summary.}
Taken together, E1 validates the local surrogates, E2 validates the end-to-end
transport-aware predictor on GPT-2, E3 validates the resulting layerwise
localization, and E5 shows that the same score can drive a useful budgeted
controller. The experiments therefore support the full pipeline of the paper:
from first-order layerwise theory, to a practical online estimator, to a
selective mitigation algorithm.

\section{LIMITATIONS}
\label{sec:limitations}

Our analysis is intentionally first-order. The theory isolates the dominant
finite-precision mechanisms arising from attention, LayerNorm, and residual
transport, and E1--E3 validate precisely this regime. The practical estimator,
however, still depends on causal surrogates for quantities such as softmax
sensitivity, downstream transport, and effective remainder magnitude. These
surrogates preserve the structure of the unified bound and work well
empirically, but they are not online certificates of exact mismatch. Likewise,
the implementation-dependent constants in the theory are treated as fixed for a
given hardware/runtime pair rather than derived analytically for every kernel.

The empirical scope is also deliberate. Our large-model experiments use GPT-2
and LayerNorm-based stabilization, so we do not claim immediate quantitative
transfer to architectures with substantially different normalization, routing,
or attention implementations. In addition, BGSS studies a narrow but controlled
intervention class---selective LayerNorm $\varepsilon$ updates under an
explicit budget---rather than the full space of numerical mitigation methods.
Accordingly, the paper supports the claim that layer-wise fragility structure
can drive useful selective stabilization, but it does not claim that BGSS is a
globally optimal controller across architectures, kernels, and runtimes.

\paragraph{Future work.}
Several extensions are especially natural. On the theory side, an important
next step is to derive sharper implementation-aware constants and to extend the
analysis to RMSNorm, grouped-query or multi-query attention, mixture-of-experts
routing, and more aggressive quantization regimes. On the algorithmic side, it
would be valuable to enlarge the intervention space beyond LayerNorm
$\varepsilon$ bumps to include selective recomputation, mixed-precision
routing, or budgeted precision promotion for specific kernels. On the empirical
side, scaling the evaluation from GPT-2 to larger contemporary language models
would test how the same layer-wise decomposition behaves under newer
architectural choices and runtime stacks.

\section{CONCLUSION}
\label{sec:conclusion}

We presented a layer-wise theory of numerical fragility in Transformers that
connects attention-side sensitivity, LayerNorm instability, and residual
transport through a unified first-order forward-error decomposition. From this
decomposition, we derived a practical causal risk estimator and a budgeted
selective stabilization rule, BGSS. The experiments support the full pipeline:
controlled sweeps verify the local mechanism statements, GPT-2 evaluation shows
that the transport-aware combined predictor tracks FP32-reference mismatch,
reference-patch attribution shows that the same score carries useful
layer-localization information, and budget-matched intervention experiments
show that it can support selective mitigation. Taken together, these results
argue that finite-precision instability in Transformers is not merely a global
numerical artifact, but a structured layer-wise phenomenon that can be
analyzed, estimated, localized, and selectively acted upon using causal
information available during execution.

% =========================
% References and Appendices
% =========================

\newpage

\bibliographystyle{apalike}
\bibliography{references_camera_ready}

\newpage

\section*{Checklist}

\begin{enumerate}

  \item For all models and algorithms presented, check if you include:
  \begin{enumerate}
    \item A clear description of the mathematical setting, assumptions, algorithm, and/or model. [Yes]
    \item An analysis of the properties and complexity (time, space, sample size) of any algorithm. [Yes]
    \item (Optional) Anonymized source code, with specification of all dependencies, including external libraries. [Yes]
  \end{enumerate}

  \item For any theoretical claim, check if you include:
  \begin{enumerate}
    \item Statements of the full set of assumptions of all theoretical results. [Yes]
    \item Complete proofs of all theoretical results. [Yes]
    \item Clear explanations of any assumptions. [Yes]     
  \end{enumerate}

  \item For all figures and tables that present empirical results, check if you include:
  \begin{enumerate}
    \item The code, data, and instructions needed to reproduce the main experimental results (either in the supplemental material or as a URL). [Yes]
    \item All the training details (e.g., data splits, hyperparameters, how they were chosen). [Yes]
    \item A clear definition of the specific measure or statistics and error bars (e.g., with respect to the random seed after running experiments multiple times). [Yes]
    \item A description of the computing infrastructure used. (e.g., type of GPUs, internal cluster, or cloud provider). [Yes]
  \end{enumerate}

  \item If you are using existing assets (e.g., code, data, models) or curating/releasing new assets, check if you include:
  \begin{enumerate}
    \item Citations of the creator If your work uses existing assets. [Yes]
    \item The license information of the assets, if applicable. [Yes]
    \item New assets either in the supplemental material or as a URL, if applicable. [Yes]
    \item Information about consent from data providers or curators. [Not Applicable]
    \item Discussion of sensible content if applicable, e.g., personally identifiable information or offensive content. [Not Applicable]
  \end{enumerate}

  \item If you used crowdsourcing or conducted research with human subjects, check if you include:
  \begin{enumerate}
    \item The full text of instructions given to participants and screenshots. [Not Applicable]
    \item Descriptions of potential participant risks, with links to Institutional Review Board (IRB) approvals if applicable. [Not Applicable]
    \item The estimated hourly wage paid to participants and the total amount spent on participant compensation. [Not Applicable]
  \end{enumerate}

\end{enumerate}

\newpage

\clearpage
\appendix
\onecolumn

\section{PROOFS AND ADDITIONAL DERIVATIONS}
\label{app:proofs}

\subsection{Floating-Point Kernel Model and Basic Lemmas}
\label{app:kernels-lemmas}

We adopt the scalar FP model $\mathrm{fl}(a\circ b)=(a\circ b)(1+\delta)$ with $|\delta|\le\epsilon_{\mathrm{mach}}$ and $\circ\in\{+,-,\times,\div\}$. For structured kernels (reductions, GEMMs, softmax, layer normalization), we use standard first-order bounds that collect rounding into implementation-dependent constants (independent of the specific model instance). We record the basic lemmas used in the proofs.

\begin{lemma}[GEMM first-order score bound]
\label{lem:score-bound}
Let $\widetilde{S}=\mathrm{fl}(QK^\top/\sqrt d)$. Then
\[
\|\widetilde{S}-S\|
\le
\epsilon_{\mathrm{mach}}
\left(
\|S\|
+
c_{\mathrm{gemm}}\frac{\|Q\|\,\|K\|}{\sqrt d}
\right)
+
O(\epsilon_{\mathrm{mach}}^2).
\]
\end{lemma}

\begin{proof}
Write $\widetilde{S}=\tfrac{1}{\sqrt d}\mathrm{fl}(QK^\top)=\tfrac{1}{\sqrt d}(QK^\top+E)$ with $\|E\|\le c_{\mathrm{gemm}}\epsilon_{\mathrm{mach}}\|Q\|\,\|K\|$. The subsequent scalar division contributes a multiplicative factor $(1+\delta)$ with $|\delta|\le\epsilon_{\mathrm{mach}}$, so
\[
\widetilde{S}
=
\frac{QK^\top+E}{\sqrt d}(1+\delta)
=
S + S\delta + \frac{E}{\sqrt d} + O(\epsilon_{\mathrm{mach}}^2).
\]
Taking norms and using triangle inequality gives the claim.
\end{proof}

\begin{lemma}[GEMM output forward bound]
\label{lem:gemm-out}
Let $\widetilde{A}=\mathrm{fl}(MV)$ for compatible matrices $M,V$. Then
\[
\widetilde{A}=MV+E_{MV},
\]
with
\[
\|E_{MV}\|
\le
c'_{\mathrm{gemm}}\epsilon_{\mathrm{mach}}\|M\|\,\|V\|_2
+ O(\epsilon_{\mathrm{mach}}^2).
\]
\end{lemma}

\begin{proof}
Write the rows of $M$ as $m_i^\top$. Standard first-order backward-error
analysis for each row-times-matrix product gives row perturbations
$\Delta m_i$ such that the computed $i$th output row satisfies
\[
\widetilde a_i^\top
=
(m_i+\Delta m_i)^\top V,
\qquad
\|\Delta m_i\|_2
\le
\gamma_{\mathrm{gemm}}\|m_i\|_2,
\]
with
\[
\gamma_{\mathrm{gemm}}
=
c'_{\mathrm{gemm}}\epsilon_{\mathrm{mach}}
+ O(\epsilon_{\mathrm{mach}}^2),
\]
where $c'_{\mathrm{gemm}}$ depends only on the GEMM kernel/runtime. Stacking the
rows defines a perturbation matrix $\Delta M$ satisfying
\[
\widetilde A
=
(M+\Delta M)V
=
MV + E_{MV},
\qquad
E_{MV}:=\Delta M\,V.
\]
Moreover,
\[
\|\Delta M\|_F^2
=
\sum_i \|\Delta m_i\|_2^2
\le
\gamma_{\mathrm{gemm}}^2
\sum_i \|m_i\|_2^2
=
\gamma_{\mathrm{gemm}}^2\|M\|_F^2.
\]
Therefore,
\[
\|E_{MV}\|_F
\le
\|\Delta M\|_F\|V\|_2
\le
\gamma_{\mathrm{gemm}}\|M\|_F\|V\|_2
=
c'_{\mathrm{gemm}}\epsilon_{\mathrm{mach}}\|M\|\,\|V\|_2
+ O(\epsilon_{\mathrm{mach}}^2),
\]
which is the claimed bound.
\end{proof}

\begin{lemma}[Softmax kernel first-order forward bound]
\label{lem:softmax-kernel}
Let $\mathcal{S}$ denote the row-wise softmax map, let
\[
\widehat{P}:=\mathcal{S}(S+\Delta S),
\]
and let $\widetilde{P}=\mathrm{fl}_{\mathrm{smx}}(S+\Delta S)$ be the
finite-precision softmax output produced from the perturbed scores
$S+\Delta S$, where $\|\Delta S\|=O(\epsilon_{\mathrm{mach}})$. Then
\[
\widetilde{P}
=
\widehat{P}+E_{\mathrm{smx}},
\]
with
\[
\|E_{\mathrm{smx}}\|
\le
c_{\mathrm{smx}}\epsilon_{\mathrm{mach}}\|\widehat{P}\|
+ O(\epsilon_{\mathrm{mach}}^2).
\]
Equivalently,
\[
\|E_{\mathrm{smx}}\|
\le
c_{\mathrm{smx}}\epsilon_{\mathrm{mach}}\|P\|
+ O(\epsilon_{\mathrm{mach}}^2),
\]
where $P=\mathcal{S}(S)$.
\end{lemma}

\begin{proof}
The first bound is the standard first-order forward model for the
implementation of the row-wise softmax kernel, with
$c_{\mathrm{smx}}$ depending only on the runtime/kernel. Since
$\mathcal{S}$ is smooth and $\|\Delta S\|=O(\epsilon_{\mathrm{mach}})$,
\[
\widehat{P}
=
\mathcal{S}(S+\Delta S)
=
P + D\mathcal{S}_S[\Delta S] + O(\|\Delta S\|^2)
=
P + O(\epsilon_{\mathrm{mach}}).
\]
Hence
\[
\|\widehat{P}\|
\le
\|P\| + O(\epsilon_{\mathrm{mach}}),
\]
and multiplying by $\epsilon_{\mathrm{mach}}$ gives
\[
\epsilon_{\mathrm{mach}}\|\widehat{P}\|
=
\epsilon_{\mathrm{mach}}\|P\|
+ O(\epsilon_{\mathrm{mach}}^2),
\]
which yields the second form.
\end{proof}

\begin{lemma}[Row-wise softmax Jacobian norm]
\label{lem:softmax-jac}
For $p=\mathrm{softmax}(s)\in\mathbb{R}^n$, $J(p)=\mathrm{Diag}(p)-pp^\top$ satisfies $0\le \|J(p)\|_2\le\tfrac12$, with the maximum approached near two-way ties.
\end{lemma}

\begin{proof}
$J(p)$ is a covariance matrix of a categorical distribution with probabilities $p$, hence PSD with operator norm bounded by the largest variance along any direction in the probability simplex. The extremum occurs when mass is split between two coordinates, giving $\|J(p)\|_2=\tfrac12$.
\end{proof}

\begin{lemma}[Closure of local first-order coefficients]
\label{lem:eff-closure}
Let $\mathcal{R}_\ell$ be a finite index set and suppose that, for each
$r\in\mathcal{R}_\ell$,
\[
\|\Delta_{\ell,r}\|
\le
\epsilon_{\mathrm{mach}}\,b_{\ell,r}
+ O(\epsilon_{\mathrm{mach}}^2).
\]
Define
\[
\Delta^{\mathrm{eff}}_\ell := \sum_{r\in\mathcal{R}_\ell}\Delta_{\ell,r},
\qquad
M^{\mathrm{eff}}_\ell := \sum_{r\in\mathcal{R}_\ell} b_{\ell,r}.
\]
Then
\[
\|\Delta^{\mathrm{eff}}_\ell\|
\le
\epsilon_{\mathrm{mach}}\,M^{\mathrm{eff}}_\ell
+ O(\epsilon_{\mathrm{mach}}^2).
\]
\end{lemma}

\begin{proof}
By triangle inequality,
\[
\|\Delta^{\mathrm{eff}}_\ell\|
\le
\sum_{r\in\mathcal{R}_\ell}
\|\Delta_{\ell,r}\|.
\]
Applying the assumed bound termwise gives
\[
\|\Delta^{\mathrm{eff}}_\ell\|
\le
\epsilon_{\mathrm{mach}}
\sum_{r\in\mathcal{R}_\ell} b_{\ell,r}
+
\sum_{r\in\mathcal{R}_\ell} O(\epsilon_{\mathrm{mach}}^2).
\]
Because $\mathcal{R}_\ell$ is finite, the last sum remains
$O(\epsilon_{\mathrm{mach}}^2)$, yielding the claim.
\end{proof}

\subsection{Proof of Theorem~\ref{thm:attn} (Self-Attention Forward Error)}
\label{app:attn-proof}

\begin{proof}
\phantom{.}
\begin{enumerate}
  \item \textbf{Pipeline decomposition.}\\
  The self-attention pipeline is $S=\tfrac{1}{\sqrt d}QK^\top$, $P=\mathrm{softmax}(S)$ (row-wise), and $A=PV$.
  First-order floating-point (FP) rounding is injected at each stage and then propagated; we bound each stage and compose.

  \item \textbf{Stage 1 — score computation (GEMM + scaling).}\\
  Let $\widetilde{S}=\mathrm{fl}(QK^\top/\sqrt d)$ and write
  $\Delta S:=\widetilde{S}-S$. By Lemma~\ref{lem:score-bound},
  \[
  \|\Delta S\|
  \le
  \epsilon_{\mathrm{mach}}
  \left(
  \|S\|
  +
  c_{\mathrm{gemm}}\frac{\|Q\|\,\|K\|}{\sqrt d}
  \right)
  +O(\epsilon_{\mathrm{mach}}^2).
  \]

  \item \textbf{Stage 2 --- softmax (Fr\'echet sensitivity + in-kernel rounding).}\\
  Let $\mathcal{S}$ denote the row-wise softmax map. Since $\mathcal{S}$ is smooth,
  \[
  \widetilde{P}
  = \mathcal{S}(S+\Delta S) + E_{\mathrm{smx}}
  = P + D\mathcal{S}_S[\Delta S] + E_{\mathrm{smx}} + O(\|\Delta S\|^2),
  \]
  where $E_{\mathrm{smx}}$ is the kernel-rounding term from
  Lemma~\ref{lem:softmax-kernel} and satisfies
  \[
  \frac{\|E_{\mathrm{smx}}\|}{\|P\|}
  \le c_{\mathrm{smx}}\epsilon_{\mathrm{mach}} + O(\epsilon_{\mathrm{mach}}^2).
  \]
  Because $D\mathcal{S}_S$ is block diagonal across rows,
  \[
  \|D\mathcal{S}_S\|_{F\to F}
  = \max_{1 \le i \le n}\|J(P_{i:})\|_2.
  \]
  By Lemma~\ref{lem:softmax-jac}, this quantity is always finite and at most
  $\tfrac12$.
  Therefore
  \[
  \frac{\|\widetilde{P}-P\|}{\|P\|}
  \le
  \frac{\|D\mathcal{S}_S\|_{F\to F}\,\|\Delta S\|}{\|P\|}
  + c_{\mathrm{smx}}\epsilon_{\mathrm{mach}}
  + O(\epsilon_{\mathrm{mach}}^2)
  \]
  \[
  \le
  \Bigl[
  c_{\mathrm{smx}}
  +
  \kappa_{\mathrm{softmax}}
  +
  c_{\mathrm{gemm}}\chi_{\mathrm{score}}
  \Bigr]\epsilon_{\mathrm{mach}}
  + O(\epsilon_{\mathrm{mach}}^2).
  \]
  \item \textbf{Stage 3 — value projection (GEMM).}\\
  Let $\widehat{A}:=\widetilde{P}V$. By Lemma~\ref{lem:gemm-out}, the finite-precision
  GEMM output satisfies
  \[
  \widetilde{A}
  =
  \widehat{A}+E_A
  =
  \widetilde{P}V+E_A,
  \]
  with
  \[
  \|E_A\|
  \le
  c'_{\mathrm{gemm}}\epsilon_{\mathrm{mach}}\|\widetilde{P}\|\,\|V\|_2
  + O(\epsilon_{\mathrm{mach}}^2).
  \]
  Therefore
  \[
  \widetilde{A}-A
  =
  (\widetilde{P}-P)V + E_A.
  \]
  By Stage~2, $\|\widetilde{P}-P\|/\|P\|=O(\epsilon_{\mathrm{mach}})$, hence
  \[
  \|\widetilde{P}\|
  \le
  \|P\|+\|\widetilde{P}-P\|
  =
  \|P\|+O(\epsilon_{\mathrm{mach}})\|P\|.
  \]
  Substituting this into the bound for $E_A$ gives
  \[
  \|E_A\|
  \le
  c'_{\mathrm{gemm}}\epsilon_{\mathrm{mach}}\|P\|\,\|V\|_2
  + O(\epsilon_{\mathrm{mach}}^2).
  \]
  Hence
  \[
  \|\widetilde{A}-A\|
  \le
  \Bigl(
  \|\widetilde{P}-P\|
  +
  c'_{\mathrm{gemm}}\epsilon_{\mathrm{mach}}\|P\|
  \Bigr)\|V\|_2
  + O(\epsilon_{\mathrm{mach}}^2).
  \]
  Dividing by $\|A\|=\|PV\|>0$ and multiplying/dividing the first term by
  $\|P\|$ yields
  \[
  \frac{\|\widetilde{A}-A\|}{\|A\|}
  \le
  \left(\frac{\|\widetilde{P}-P\|}{\|P\|}
  + c'_{\mathrm{gemm}}\epsilon_{\mathrm{mach}}\right)
  \frac{\|P\|\,\|V\|_2}{\|A\|}
  \]
  \[
  + O(\epsilon_{\mathrm{mach}}^2)
  =
  \left(\frac{\|\widetilde{P}-P\|}{\|P\|}
  + c'_{\mathrm{gemm}}\epsilon_{\mathrm{mach}}\right)\kappa_{\mathrm{val}}
  \]
  \[
  + O(\epsilon_{\mathrm{mach}}^2).
  \]

  \item \textbf{Final combination.}\\
  Substitute the Stage~2 bound into Stage~3:
    \[
    \begin{split}
    \|\widetilde{A}-A\|
    \;\le\;& \Big[
      c_{\mathrm{smx}}
      +
      \kappa_{\mathrm{softmax}}
      +
      c_{\mathrm{gemm}}\chi_{\mathrm{score}}
      \\
    & + c'_{\mathrm{gemm}}
    \Big]\epsilon_{\mathrm{mach}}\;\|P\|\,\|V\|_2
    + O(\epsilon_{\mathrm{mach}}^2).
    \end{split}
    \]
  If additionally $\|A\|>0$, dividing by $\|A\|$ gives
    \[
    \begin{split}
    \frac{\|\widetilde{A}-A\|}{\|A\|}
    \;\le\;& \Big[
      c_{\mathrm{smx}}
      +
      \kappa_{\mathrm{softmax}}
      +
      c_{\mathrm{gemm}}\chi_{\mathrm{score}}
      \\
    & + c'_{\mathrm{gemm}}
    \Big]\epsilon_{\mathrm{mach}}\;\kappa_{\mathrm{val}}
    + O(\epsilon_{\mathrm{mach}}^2).
    \end{split}
    \]
  This is exactly the main-text statement of Theorem~\ref{thm:attn}.
\end{enumerate}
\end{proof}

\subsection{Proof of Theorem~\ref{thm:residual} and Corollary~\ref{cor:depth}}
\label{app:residual-proof}

\begin{proof}
\leavevmode
\begin{enumerate}
  \item[(1)] \textbf{Theorem~\ref{thm:residual}.} Let $F:=J_f$ with $\|F\|_2<1$ and $T:=I+F$.
  \begin{enumerate}
    \item[(a)] \emph{Forward norm.} By subadditivity, $\|T\|_2=\|I+F\|_2\le \|I\|_2+\|F\|_2=1+\|F\|_2$.
    \item[(b)] \emph{Inverse norm.} Since $\|F\|_2<1$, $T$ is invertible and
    $(I+F)^{-1}=\sum_{k=0}^{\infty}(-F)^k$; thus $\|T^{-1}\|_2\le \sum_{k=0}^{\infty}\|F\|_2^k=(1-\|F\|_2)^{-1}$.
    \item[(c)] \emph{Combine.} Therefore
    \[
    \kappa(T)=\|T\|_2\,\|T^{-1}\|_2 \;\le\; \frac{1+\|F\|_2}{1-\|F\|_2}.
    \]
  \end{enumerate}

  \item[(2)] \textbf{Corollary~\ref{cor:depth} (depth-wise relaxation).}
  Consider a stack of residual blocks with Jacobians $\{F_\ell:=J_{f_\ell}\}_{\ell=1}^L$ and $\rho_\ell:=\|F_\ell\|_2<1$.
  Let $T_\ell:=I+F_\ell$. For the linearized composition $T:=\prod_{\ell=1}^L T_\ell$ we have
  \begin{enumerate}
    \item[(a)] \emph{Submultiplicativity of $\kappa$.}
    For any compatible $A,B$, $\kappa(AB)\le \kappa(A)\,\kappa(B)$ since
    $\|AB\|_2\le \|A\|_2\|B\|_2$ and $\|(AB)^{-1}\|_2=\|B^{-1}A^{-1}\|_2\le \|B^{-1}\|_2\|A^{-1}\|_2$.
    \item[(b)] \emph{Product bound.} Applying (a) iteratively and Theorem~\ref{thm:residual} to each $T_\ell$,
    \[
    \kappa\!\Big(\prod_{\ell=1}^L T_\ell\Big)
    \;\le\; \prod_{\ell=1}^L \kappa(T_\ell)
    \;\le\; \prod_{\ell=1}^L \frac{1+\rho_\ell}{1-\rho_\ell}.
    \]
    \item[(c)] \emph{First-order relaxation (used in main text).}
    Since $\frac{1+\rho}{1-\rho}=1+2\rho+O(\rho^2)$ for $\rho\in[0,1)$, a first-order \emph{relaxation model}
    replaces $\frac{1+\rho_\ell}{1-\rho_\ell}$ by $(1+\rho_\ell)$. This captures the attenuation effect of residuals
    while avoiding an overly pessimistic multiplicative growth; note this replacement is an approximation (not an upper bound).
  \end{enumerate}
\end{enumerate}
\end{proof}

\subsection{Proof of Proposition~\ref{prop:ln} (Normalization-path forward error and monotonicity)}
\label{app:ln-proof}

\begin{proof}
\phantom{.}
\begin{enumerate}
  \item \textbf{Setup and notation.}
  Let $d=d_{\text{model}}$, $m=\mu(x)=\tfrac{1}{d}\mathbf{1}^\top x$,
  $v=\sigma^2(x)=\tfrac{1}{d}\|x-m\mathbf{1}\|_2^2$,
  $C:=I-\tfrac{1}{d}\mathbf{1}\mathbf{1}^\top$, and
  \[
  c:=x-m\mathbf{1}=Cx,\qquad
  \alpha:=(v+\varepsilon)^{-1/2},\qquad
  g(x):=\frac{x-m\mathbf{1}}{\sqrt{v+\varepsilon}}=\alpha c,
  \qquad
  Z^{\mathrm{LN}}_\varepsilon(x):=\mathrm{Diag}(\gamma)\,g(x).
  \]
  The full LayerNorm output is $\mathrm{LN}(x)=Z^{\mathrm{LN}}_\varepsilon(x)+\beta$,
  but Proposition~\ref{prop:ln} isolates the $\varepsilon$-dependent normalization
  path $Z^{\mathrm{LN}}_\varepsilon$. The final bias addition is
  $\varepsilon$-independent and is absorbed into the effective remainder term in
  Theorem~\ref{thm:unified}.

  \item \textbf{Centering stage.}
  Let $\widetilde{m}=\mathrm{fl}(\tfrac{1}{d}\mathbf{1}^\top x)=m+\delta_m$.
  Standard first-order reduction bounds give
  \[
  |\delta_m|
  \le
  c_m\epsilon_{\mathrm{mach}}\frac{\|x\|_2}{\sqrt d}
  + O(\epsilon_{\mathrm{mach}}^2)
  \]
  for a kernel- and dimension-dependent constant $c_m>0$. Let
  \[
  \widetilde{c}:=\mathrm{fl}(x-\widetilde{m}\mathbf{1})=c+e_c.
  \]
  The subtraction stage satisfies
  \[
  e_c=-\delta_m\mathbf{1}+e_{\mathrm{sub}},
  \]
  with
  \[
  \|e_{\mathrm{sub}}\|_2
  \le
  c_{\mathrm{sub}}\epsilon_{\mathrm{mach}}
  \bigl(\|x\|_2+\sqrt d\,|\widetilde{m}|\bigr)
  + O(\epsilon_{\mathrm{mach}}^2).
  \]
  Since $|m|\le \|x\|_2/\sqrt d$ and
  $|\widetilde{m}|=|m|+O(\epsilon_{\mathrm{mach}})\|x\|_2/\sqrt d$, there exists
  $c_c>0$ such that
  \[
  \|e_c\|_2
  \le
  c_c\epsilon_{\mathrm{mach}}\|x\|_2
  + O(\epsilon_{\mathrm{mach}}^2).
  \]

  \item \textbf{Variance and reciprocal-square-root stage.}
  Let
  \[
  \widetilde{v}:=\mathrm{fl}\!\left(\frac{1}{d}\|\widetilde{c}\|_2^2\right)
  = v+\delta_v.
  \]
  Decompose
  \[
  \delta_v
  =
  \frac{\|\widetilde{c}\|_2^2-\|c\|_2^2}{d}
  + e_v,
  \]
  where $e_v$ collects first-order reduction/multiplication rounding in the
  squared-norm computation. Using $\widetilde{c}=c+e_c$ and
  $\|c\|_2=\sqrt d\,\sigma(x)$,
  \[
  \left|
  \frac{\|\widetilde{c}\|_2^2-\|c\|_2^2}{d}
  \right|
  \le
  \frac{2}{d}\|c\|_2\|e_c\|_2
  + O(\epsilon_{\mathrm{mach}}^2)
  \le
  \frac{2c_c}{\sqrt d}\epsilon_{\mathrm{mach}}\sigma(x)\|x\|_2
  + O(\epsilon_{\mathrm{mach}}^2).
  \]
  Standard first-order reduction bounds also give
  \[
  |e_v|
  \le
  c_v\epsilon_{\mathrm{mach}}\,v
  + O(\epsilon_{\mathrm{mach}}^2)
  \]
  for some $c_v>0$. Therefore there exist kernel- and dimension-dependent
  constants $\bar a_v,\bar b_v>0$ such that
  \[
  |\delta_v|
  \le
  \epsilon_{\mathrm{mach}}
  \bigl(\bar a_v\,\sigma(x)\|x\|_2+\bar b_v\,v\bigr)
  + O(\epsilon_{\mathrm{mach}}^2).
  \]
  Let
  \[
  \widetilde{\alpha}
  :=\mathrm{fl}\!\bigl((\widetilde{v}+\varepsilon)^{-1/2}\bigr)
  =\alpha+\delta_\alpha.
  \]
  By the mean value theorem for $t\mapsto t^{-1/2}$ and first-order
  reciprocal-square-root rounding,
  \[
  |\delta_\alpha|
  \le
  \frac{|\delta_v|}{2(v+\varepsilon)^{3/2}}
  +
  c_{\mathrm{rsqrt}}\epsilon_{\mathrm{mach}}\frac{1}{\sqrt{v+\varepsilon}}
  + O(\epsilon_{\mathrm{mach}}^2).
  \]

  \item \textbf{Normalization core.}
  The computed normalized core is
  \[
  \widetilde{g}(x)
  =
  \mathrm{fl}(\widetilde{\alpha}\widetilde{c})
  =
  \alpha c+\alpha e_c+\delta_\alpha c+e_{\mathrm{mul}}
  + O(\epsilon_{\mathrm{mach}}^2),
  \]
  where pointwise multiplication gives
  \[
  \|e_{\mathrm{mul}}\|_2
  \le
  c_{\mathrm{mul}}\epsilon_{\mathrm{mach}}
  \frac{\|x\|_2}{\sqrt{v+\varepsilon}}
  + O(\epsilon_{\mathrm{mach}}^2).
  \]
  Hence
  \[
  \|\widetilde{g}(x)-g(x)\|_2
  \le
  \frac{\|e_c\|_2}{\sqrt{v+\varepsilon}}
  +
  \|c\|_2\,|\delta_\alpha|
  +
  c_{\mathrm{mul}}\epsilon_{\mathrm{mach}}
  \frac{\|x\|_2}{\sqrt{v+\varepsilon}}
  + O(\epsilon_{\mathrm{mach}}^2).
  \]
  Substituting the bounds above yields a linear combination of
  \[
  \frac{\|x\|_2}{\sqrt{v+\varepsilon}},
  \qquad
  \frac{v\,\|x\|_2}{(v+\varepsilon)^{3/2}},
  \]
  with kernel- and dimension-dependent coefficients. Here we used
  \[
  \|c\|_2=\sqrt d\,\sigma(x)\le \|x\|_2,
  \qquad
  \|c\|_2\sigma(x)=\sqrt d\,v,
  \]
  so the $\delta_v$ terms are also controlled by the second quantity above
  after absorbing dimension-only factors into the constants. Since
  \[
  \frac{1}{\sqrt{v+\varepsilon}}
  \le
  \frac{\varepsilon+2v}{(v+\varepsilon)^{3/2}},
  \qquad
  \frac{v}{(v+\varepsilon)^{3/2}}
  \le
  \frac{\varepsilon+2v}{(v+\varepsilon)^{3/2}},
  \]
  there exists a kernel- and dimension-dependent constant
  $\bar a_{\mathrm{ln}}>0$ such that
  \[
  \|\widetilde{g}(x)-g(x)\|_2
  \le
  \epsilon_{\mathrm{mach}}\,
  \bar a_{\mathrm{ln}}
  \frac{\varepsilon+2v}{(v+\varepsilon)^{3/2}}
  \|x\|_2
  + O(\epsilon_{\mathrm{mach}}^2).
  \]

  \item \textbf{Finite-precision scaling by \texorpdfstring{$\gamma$}{gamma}.}
  Let $\widetilde{Z}^{\mathrm{LN}}_\varepsilon(x)$ denote the finite-precision
  result of the pointwise multiplication by $\gamma$. Standard pointwise
  multiplication bounds give
  \[
  \widetilde{Z}^{\mathrm{LN}}_\varepsilon(x)-Z^{\mathrm{LN}}_\varepsilon(x)
  =
  \mathrm{Diag}(\gamma)\bigl(\widetilde{g}(x)-g(x)\bigr)
  + e_\gamma
  + O(\epsilon_{\mathrm{mach}}^2),
  \]
  where
  \[
  \|e_\gamma\|_2
  \le
  c_\gamma\,\epsilon_{\mathrm{mach}}\,
  \|\mathrm{Diag}(\gamma)\|_2\,\|g(x)\|_2
  + O(\epsilon_{\mathrm{mach}}^2)
  \]
  for a kernel-dependent constant $c_\gamma>0$. Since
  \[
  \|g(x)\|_2
  =
  \frac{\|c\|_2}{\sqrt{v+\varepsilon}}
  \le
  \frac{\varepsilon+2v}{(v+\varepsilon)^{3/2}}
  \|x\|_2,
  \]
  so
  \[
  \|\widetilde{Z}^{\mathrm{LN}}_\varepsilon(x)-Z^{\mathrm{LN}}_\varepsilon(x)\|_2
  \le
  \epsilon_{\mathrm{mach}}\,
  \|\mathrm{Diag}(\gamma)\|_2\,
  \frac{\varepsilon+2v}{(v+\varepsilon)^{3/2}}
  \bigl(\bar a_{\mathrm{ln}}+c_\gamma\bigr)\|x\|_2
  + O(\epsilon_{\mathrm{mach}}^2).
  \]
  Define
  \[
  a_{\mathrm{ln}}:=\bar a_{\mathrm{ln}}+c_\gamma.
  \]

  \item \textbf{Absolute forward error bound.}
  With $a_{\mathrm{ln}}:=\bar a_{\mathrm{ln}}+c_\gamma$, Step 5 gives
  \[
  \|\widetilde{Z}^{\mathrm{LN}}_\varepsilon(x)-Z^{\mathrm{LN}}_\varepsilon(x)\|
  \le
  \epsilon_{\mathrm{mach}}
  \|\mathrm{Diag}(\gamma)\|_2
  \frac{\varepsilon+2\sigma^2(x)}{(\sigma^2(x)+\varepsilon)^{3/2}}
  \bigl(a_{\mathrm{ln}}\|x\|_2\bigr)
  + O(\epsilon_{\mathrm{mach}}^2),
  \]
  which is exactly \eqref{eq:ln-abs-bound} with
  \[
  M_{\mathrm{LN}}(x,\varepsilon)
  :=
  \|\mathrm{Diag}(\gamma)\|_2
  \frac{\varepsilon+2\sigma^2(x)}{(\sigma^2(x)+\varepsilon)^{3/2}}
  \bigl(a_{\mathrm{ln}}\|x\|_2\bigr).
  \]

  \item \textbf{Relative forward error bound.}
  If additionally $\|Z^{\mathrm{LN}}_\varepsilon(x)\|>0$, dividing the previous
  inequality by $\|Z^{\mathrm{LN}}_\varepsilon(x)\|$ gives
  \[
  \frac{\|\widetilde{Z}^{\mathrm{LN}}_\varepsilon(x)-Z^{\mathrm{LN}}_\varepsilon(x)\|}{\|Z^{\mathrm{LN}}_\varepsilon(x)\|}
  \le
  \epsilon_{\mathrm{mach}}
  \frac{M_{\mathrm{LN}}(x,\varepsilon)}{\|Z^{\mathrm{LN}}_\varepsilon(x)\|}
  + O(\epsilon_{\mathrm{mach}}^2),
  \]
  which is exactly the stated coefficient $C_{\mathrm{LN}}(x,\varepsilon)$.

  \item \textbf{Monotonicity of the $\varepsilon$-dependent factor.}
  For fixed $v=\sigma^2(x)$, define
  \[
  f_v(\varepsilon)
  :=
  \frac{\varepsilon+2v}{(v+\varepsilon)^{3/2}}.
  \]
  Then
  \[
  \frac{d}{d\varepsilon}f_v(\varepsilon)
  =
  -\frac{\varepsilon+4v}{2(v+\varepsilon)^{5/2}}
  < 0
  \qquad (\varepsilon>0),
  \]
  so $f_v$ is strictly decreasing in $\varepsilon$. This proves the monotonicity claim in Proposition~\ref{prop:ln}.
\end{enumerate}
\end{proof}

\subsection{Proof of Proposition~\ref{prop:bgss} (Frozen-scale monotonicity)}
\label{app:bgss-proof}

\begin{proof}
\phantom{.}
\begin{enumerate}
  \item \textbf{Derivative of the frozen-scale LayerNorm magnitude.}\\
  Write
  \[
  \alpha_{\ell,t}
  :=
  \|\mathrm{Diag}(\gamma_\ell)\|_2
  \bigl(a_{\ell,t}\,\|x_{\ell,t}\|_2\bigr),
  \]
  so that
  \[
  M_{\mathrm{LN},\ell,t}^{\mathrm{frz}}(\varepsilon)
  =
  \alpha_{\ell,t}
  \frac{\varepsilon+2v_{\ell,t}}{(v_{\ell,t}+\varepsilon)^{3/2}}.
  \]
  Because $a_{\ell,t}>0$, $\|\mathrm{Diag}(\gamma_\ell)\|_2\ge 0$, and
  $\|x_{\ell,t}\|_2\ge 0$, we have
  $\alpha_{\ell,t}\ge 0$. Differentiating gives
  \[
  \frac{d}{d\varepsilon}M_{\mathrm{LN},\ell,t}^{\mathrm{frz}}(\varepsilon)
  =
  -\frac{\alpha_{\ell,t}}{2}
  \frac{\varepsilon+4v_{\ell,t}}{(v_{\ell,t}+\varepsilon)^{5/2}},
  \]
  which is exactly \eqref{eq:bgss-derivative}. Since $\alpha_{\ell,t}\ge 0$,
  $v_{\ell,t}\ge 0$, and $\varepsilon>0$, the derivative is nonpositive, so
  $M_{\mathrm{LN},\ell,t}^{\mathrm{frz}}$ is nonincreasing. If
  $\|\mathrm{Diag}(\gamma_\ell)\|_2\,\|x_{\ell,t}\|_2>0$, then $\alpha_{\ell,t}>0$, hence the
  derivative is strictly negative for every $\varepsilon>0$, and
  $M_{\mathrm{LN},\ell,t}^{\mathrm{frz}}$ is strictly decreasing.

  \item \textbf{Monotonicity of the frozen-scale layer contribution.}\\
  Define the $\varepsilon$-independent factor
  \[
  D_{\ell,t}
  :=
  \frac{1}{\|X_{L,t}\|}
  \prod_{k=\ell+1}^{L}(1+\rho_{k,t}).
  \]
  Since $\|X_{L,t}\|>0$ and $\rho_{k,t}=\|J_{f_k}(X_{k-1,t})\|_2\ge 0$, we have
  $D_{\ell,t}\ge 0$. Therefore
  \[
  G_{\ell,t}^{\mathrm{frz}}(\varepsilon)
  =
  \Bigl(
  M^{\mathrm{eff}}_{\ell,t}
  +
  A_{\ell,t}
  +
  M_{\mathrm{LN},\ell,t}^{\mathrm{frz}}(\varepsilon)
  \Bigr)D_{\ell,t}.
  \]
  Because all remaining factors are fixed with respect to
  $\varepsilon$, the
  nonincreasingness of $M_{\mathrm{LN},\ell,t}^{\mathrm{frz}}$ implies that
  $G_{\ell,t}^{\mathrm{frz}}$ is nonincreasing in $\varepsilon$.

  \item \textbf{Exact reduction formula.}\\
  For any $\varepsilon'\ge \varepsilon$, subtracting the two frozen-scale scores yields
  \[
  \begin{split}
  G_{\ell,t}^{\mathrm{frz}}(\varepsilon) - G_{\ell,t}^{\mathrm{frz}}(\varepsilon')
  =\;&
  \frac{1}{\|X_{L,t}\|}
  \Bigl(
  M_{\mathrm{LN},\ell,t}^{\mathrm{frz}}(\varepsilon)
  -
  M_{\mathrm{LN},\ell,t}^{\mathrm{frz}}(\varepsilon')
  \Bigr)
  \\
  &\times
  \prod_{k=\ell+1}^{L}(1+\rho_{k,t}),
  \end{split}
  \]
  because the $M^{\mathrm{eff}}_{\ell,t}$ and $A_{\ell,t}$ terms cancel exactly. This is
  \eqref{eq:bgss-reduction}.
\end{enumerate}
\end{proof}

\subsection*{Appendix: $\varepsilon$-bump policy (as implemented)}
\begin{algorithmH}
\caption{LayerNorm $\varepsilon$-bump triggered by the practical regime indicator}
\label{alg:eps-bump}
\DontPrintSemicolon
\KwIn{target $\rho_*\in(0,1)$, check interval $t{=}5$, floor/cap $(\varepsilon_{\min}{=}10^{-6},\;\varepsilon_{\max}\in\{5\!\times\!10^{-3},10^{-2}\})$}
\For{training step $s=1,2,\dots$}{
  \If{$s \bmod t = 0$}{
    sample a 16-example minibatch; \ForEach{LN layer with current $\varepsilon$}{
      compute $\rho_{\mathrm{LN}}(\varepsilon)=\frac{\sigma^2(x)}{\varepsilon}\,d_{\text{model}}\,\epsilon_{\mathrm{mach}}$ \;
      \If{$\rho_{\mathrm{LN}}(\varepsilon)<1$}{
        $\varepsilon_{\mathrm{cand}}\gets
          \dfrac{\mathrm{median}(\sigma^2(x))\,d_{\text{model}}\,\epsilon_{\mathrm{mach}}}{\rho_*}$ \;
        $\varepsilon_{\mathrm{new}}\gets \mathrm{clip}(\varepsilon_{\mathrm{cand}},\varepsilon_{\min},\varepsilon_{\max})$ \;
        \If{$\varepsilon_{\mathrm{new}}>\varepsilon$}{ set $\varepsilon\leftarrow \varepsilon_{\mathrm{new}}$ \tcp*{monotone increase} }
      }
    }
  }
}
\end{algorithmH}

\subsection{Proof of Theorem~\ref{thm:unified} (Unified Forward Stability)}
\label{app:unified-proof}

\begin{proof}
\phantom{.}
\begin{enumerate}
  \item \textbf{Linearization and one-step recursion.}\\
  Let $X_\ell$ be the exact hidden after layer $\ell$, with $X_0$ the model input, and let $\widetilde{X}_\ell$ be its FP counterpart,
  and $E_\ell=\widetilde{X}_\ell-X_\ell$. For the pre-LN residual block $x\mapsto x+f_\ell(x)$,
  write $F_\ell:=J_{f_\ell}(X_{\ell-1})$ and $\rho_\ell=\|F_\ell\|_2$. With $E_0=0$, a first-order
  perturbation analysis gives, for each layer $\ell=1,\dots,L$,
  \[
  E_\ell=(I+F_\ell)E_{\ell-1}+\Delta_\ell+O(\epsilon_{\mathrm{mach}}^2),
  \]
  where $\Delta_\ell$ is the layer-$\ell$ local first-order forward error (from MHSA, LN, FFN).

  \item \textbf{Exact local layer coefficient.}\\
  Decompose the layer-local first-order error as
  \[
  \Delta_\ell
  =
  \Delta^{\mathrm{eff}}_\ell + \Delta^{\mathrm{attn}}_\ell + \Delta^{\mathrm{LN}}_\ell
  + O(\epsilon_{\mathrm{mach}}^2),
  \]
  where $\Delta^{\mathrm{eff}}_\ell$ collects the remaining non-attention and
  non-normalization-path
  first-order perturbations. By \eqref{eq:eff-module} and
  Lemma~\ref{lem:eff-closure},
  \[
  \|\Delta^{\mathrm{eff}}_\ell\|
  \le \epsilon_{\mathrm{mach}}\,M^{\mathrm{eff}}_\ell + O(\epsilon_{\mathrm{mach}}^2).
  \]
  Let
  \[
  M^{\mathrm{LN}}_\ell := M_{\mathrm{LN}}(x_\ell,\varepsilon_\ell)
  \]
  denote the exact absolute LayerNorm normalization-path magnitude from
  Proposition~\ref{prop:ln}. By Theorem~\ref{thm:attn},
  \[
  \|\Delta^{\mathrm{attn}}_\ell\|
  \le \epsilon_{\mathrm{mach}}\,A_\ell + O(\epsilon_{\mathrm{mach}}^2),
  \]
  and by Proposition~\ref{prop:ln},
  \[
  \|\Delta^{\mathrm{LN}}_\ell\|
  \le \epsilon_{\mathrm{mach}}\,M^{\mathrm{LN}}_\ell
  + O(\epsilon_{\mathrm{mach}}^2).
  \]
  Therefore triangle inequality gives
  \[
  \|\Delta_\ell\|
  \le
  \epsilon_{\mathrm{mach}}
  \Bigl[
  M^{\mathrm{eff}}_\ell
  + A_\ell
  + M^{\mathrm{LN}}_\ell
  \Bigr]
  + O(\epsilon_{\mathrm{mach}}^2)
  \]
  \[
  = \epsilon_{\mathrm{mach}} M_\ell + O(\epsilon_{\mathrm{mach}}^2).
  \]

  \item \textbf{Residual relaxation (downstream amplification).}\\
  Iterating the recursion from Step 1 and using $\|I+F_k\|_2\le 1+\rho_k$ yields
  \[
  \|E_L\|
  \;\le\;
  \sum_{\ell=1}^{L}
  \Bigl(\|\Delta_\ell\|\prod_{k=\ell+1}^{L}\|I+F_k\|_2\Bigr)
  \;+\; O(\epsilon_{\mathrm{mach}}^2)\,\|X_L\|
  \]
  \[
  \;\le\;
  \sum_{\ell=1}^{L}
  \Bigl(\|\Delta_\ell\|\prod_{k=\ell+1}^{L}(1+\rho_k)\Bigr)
  \;+\; O(\epsilon_{\mathrm{mach}}^2)\,\|X_L\|.
  \]

  \item \textbf{Normalization and collection of terms.}\\
  By Step 2, $\|\Delta_\ell\|\le \epsilon_{\mathrm{mach}} M_\ell + O(\epsilon_{\mathrm{mach}}^2)$.
  Any benign inter-layer norm ratios used to rewrite the local absolute
  magnitudes relative to $\|X_L\|$ are absorbed into the same layer-dependent
  first-order constants already defining $M_\ell$.
  Substitute this into Step 3 and divide by $\|X_L\|$:
  \[
  \frac{\|E_L\|}{\|X_L\|}
  \le
  \epsilon_{\mathrm{mach}}
  \sum_{\ell=1}^{L}
  \frac{M_\ell}{\|X_L\|}
  \prod_{k=\ell+1}^{L}(1+\rho_k)
  +
  O(\epsilon_{\mathrm{mach}}^2),
  \]
  which is exactly \eqref{eq:unified}.

  \item \textbf{Remainder.}\\
  The $O(\epsilon_{\mathrm{mach}}^2)$ term collects mixed higher-order interactions across kernels and layers.
\end{enumerate}
\end{proof}

\subsection{Estimation Details and Complexity Notes}
\label{app:estimation}

\paragraph{Softmax sensitivity.}
Estimate $\|D\mathcal{S}_S\|_{F\to F}=\max_i\|J(P_{i:})\|_2$ by $2$--$3$ steps of power iteration on sampled rows, then multiply by $\|S\|/\|P\|$ to obtain $\kappa_{\mathrm{softmax}}$. For multi-head, aggregate by max (or sum for a conservative predictor).

\paragraph{Score transport and value conditioning.}
Estimate $\chi_{\mathrm{score}}$ from the same sampled softmax differential together with the factor $\|Q\|\,\|K\|/(\|P\|\sqrt d)$. When $\|S\|>0$, one may equivalently form the auxiliary diagnostic $\kappa_{\mathrm{score}}$ and use $\chi_{\mathrm{score}}=\kappa_{\mathrm{softmax}}\kappa_{\mathrm{score}}$. For the value path in the layer-wise score, estimate the absolute attention magnitude factor $\|P\|\,\|V\|_2$ using the monitored tensors $\widehat{P}$ and $\widehat{V}$, for example via direct norm computation or a truncated-SVD estimate of $\|\widehat{V}\|_2$. When $\|A\|>0$, one may additionally record the relative diagnostic $\kappa_{\mathrm{val}}=\|P\|\,\|V\|_2/\|A\|$.

\paragraph{Residual transport.}
Estimate each downstream factor $\widehat{\rho}_{k,t}$ causally from the monitored
pass. Two practical options are: (i) $1$--$2$ power-iteration steps using JVP/VJP
access to estimate $\|J_{f_k}(X_{k-1,t})\|_2$, or (ii) a conservative blockwise
operator-norm surrogate built from the constituent linear maps and pointwise
nonlinearities inside the residual branch. Either choice preserves the role of
the downstream transport factor in $\widehat{G}_{\ell,t}$.

\paragraph{Effective remainder magnitude.}
Construct $\widehat{M}^{\mathrm{eff}}_{\ell,t}$ by aggregating the remaining
non-attention, non-normalization-path operations in layer $\ell$, such as output
projections, FFN maps, and the $\varepsilon$-independent LayerNorm tail. In
practice, we use causal norm-based surrogates of their first-order output
magnitudes and sum them so that $\widehat{M}^{\mathrm{eff}}_{\ell,t}$ preserves
the additive local decomposition of the exact magnitude $M^{\mathrm{eff}}_{\ell,t}$.

\paragraph{Kernel constants.}
$c_{\mathrm{gemm}},c'_{\mathrm{gemm}},c_{\mathrm{smx}},c_{\mathrm{ln}}$ are treated as fixed per hardware/runtime; they factor into the absolute constants in Theorems~\ref{thm:attn} and \ref{thm:unified}.

\begin{remark}[Dropout]\label{rem:dropout}
With inverted dropout $\mathrm{Drop}_p(z)=(M/p)\odot z$, 
$\|J_{\mathrm{Drop}_p\circ f}(x)\|_2\le (1/p)\|J_f(x)\|_2$; hence per-layer factors in our bounds
scale by at most $1/p$. Post-softmax attention dropout multiplies both $J$ and $P$ by $(1/p)$, 
leaving $\kappa_{\mathrm{softmax}}$ non-increasing under a fixed mask.
\end{remark}

\begin{lemma}[Attention dropout does not worsen $\kappa_{\mathrm{softmax}}$]
If post-softmax dropout is applied with a fixed mask, $\kappa'_{\mathrm{softmax}}\le \kappa_{\mathrm{softmax}}$.
\emph{Sketch.} $J(P')=\mathrm{Diag}(M/p)J(P)$ and $\|P'\|=\|(M/p)\odot P\|$ so the $(1/p)$ factors cancel in the ratio.
\end{lemma}

\subsection{Additional Experimental Details}
\label{app:exp-details}

\paragraph{Protocols and metrics.}
E1 consists of synthetic controlled sweeps with $20$ attention configurations,
$10$ LayerNorm $\varepsilon$ values, and $8$ residual-gain values. E2 uses the
GPT-2 checkpoint \citep{Radford2019language} on WikiText-103
\citep{Merity2016pointer} validation with BF16/FP16 monitored passes, sequence
lengths $\{128,512,1024\}$, and $3$ seeds, for $18$ completed runs and $96$
monitored windows per run. E3 reuses these E2 runs, selects the top-$8$
highest-mismatch windows per run, and evaluates single-layer FP32
reference-patch effects against the same manual forward path; the mean
recompute gap is
$9.24\times 10^{-7}$. E5 uses stable FP32-master / FP16-shadow training on
WikiText-2 train with sequence length $256$, $256$ monitored steps, $3$ seeds,
and a maximum of $24$ controller actions for the budget-matched policies.

\paragraph{Implementation, infrastructure, and assets.}
The released code is written for Python $3.10+$ and uses PyTorch, Hugging Face
\texttt{transformers}, Hugging Face \texttt{datasets}, \texttt{matplotlib}, and
a Bash-compatible orchestration script. The public repository includes default
configuration files specifying the model, dataset splits, precisions, sequence
lengths, seeds, monitoring parameters, and BGSS controller hyperparameters used
for the reported experiments. Experiments were run on NVIDIA H100 accelerators.
The external assets are the Hugging Face \texttt{gpt2} checkpoint and the
WikiText dataset accessed through Hugging Face \texttt{datasets}. The
\texttt{gpt2} checkpoint is distributed under the MIT license, and WikiText is
distributed under the Creative Commons Attribution-ShareAlike license
(CC BY-SA 4.0). The official code release is provided at the URL in
Section~\ref{sec:introduction} and is released under the MIT license.

\paragraph{Secondary quantitative summaries.}
For E2, the transport-aware combined predictor improves Pearson correlation over
the no-transport ablation by a mean of $0.163$ and Spearman correlation by a
mean of $0.139$. For E3, the mean top-$1$ hit rate is $0.264$, with mean top-$3$
and top-$5$ overlaps of $0.505$ and $0.622$, respectively. For E5, BGSS and the
budget-matched baselines use identical average protected layer-steps ($2742$).
Relative to the random same-budget controller, BGSS reduces mean onset events
by $1.0$ and reduces mean max mismatch by more than a factor of $2.7$; relative
to the risk-only same-budget controller, BGSS preserves the same mean onset
event count while reducing mean max mismatch from $5.71\times 10^{-3}$ to
$3.14\times 10^{-3}$.

\begin{figure}[H]
\centering
\begin{minipage}[t]{0.48\linewidth}
  \centering
  \includegraphics[width=\linewidth]{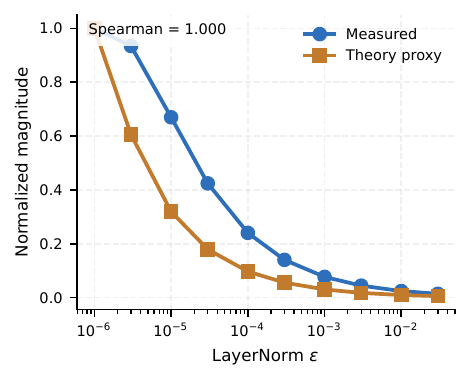}
\end{minipage}\hfill
\begin{minipage}[t]{0.48\linewidth}
  \centering
  \includegraphics[width=\linewidth]{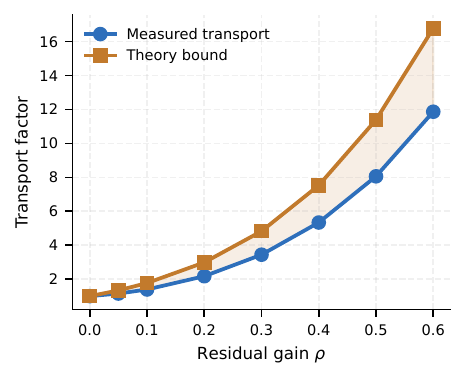}
\end{minipage}
\caption{Additional E1 controlled sweeps. Left: the measured LayerNorm
normalization-path change and the causal proxy both decrease monotonically as
$\varepsilon$ grows, matching Proposition~\ref{prop:ln}. Right: the measured
downstream amplification remains below the residual transport bound throughout
the tested $\rho$-range, matching Theorem~\ref{thm:residual} and
Corollary~\ref{cor:depth}.}
\label{fig:app:e1_pair}
\end{figure}

\begin{figure}[H]
\centering
\begin{minipage}[t]{0.45\linewidth}
  \centering
  \includegraphics[width=\linewidth]{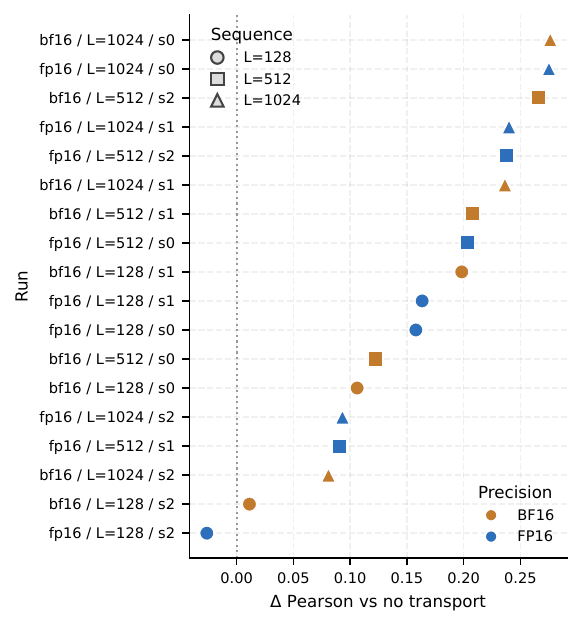}
\end{minipage}\hfill
\begin{minipage}[t]{0.45\linewidth}
  \centering
  \includegraphics[width=\linewidth]{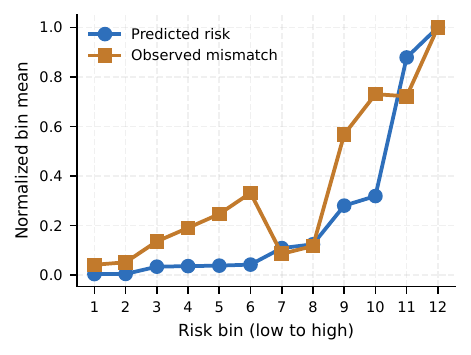}
\end{minipage}
\caption{Additional E2 views. Left: run-wise transport gain, showing that most
runs improve in Pearson correlation when downstream transport is added to the
local predictor. Right: a binned trend view in which windows sorted by the
predicted risk exhibit increasing mismatch in the high-risk bins.}
\label{fig:app:e2_pair}
\end{figure}

\begin{figure}[H]
\centering
\begin{minipage}[t]{0.48\linewidth}
  \centering
  \includegraphics[width=\linewidth]{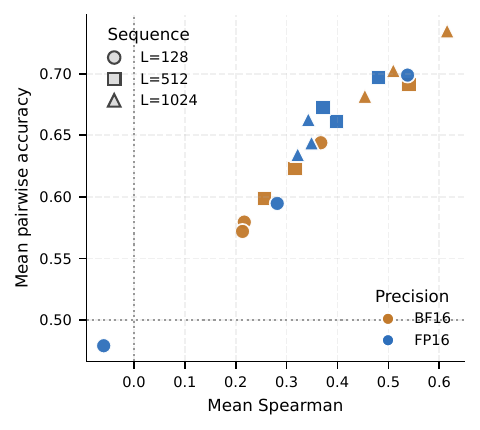}
\end{minipage}\hfill
\begin{minipage}[t]{0.48\linewidth}
  \centering
  \includegraphics[width=\linewidth]{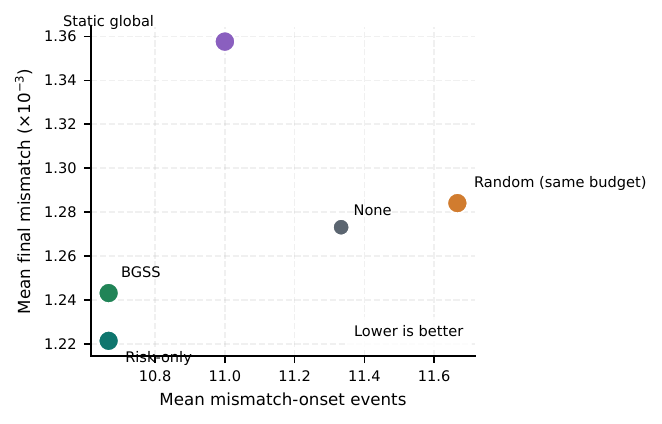}
\end{minipage}
\caption{Additional E3 and E5 summaries. Left: run-level E3 fidelity, showing
that the attribution proxy remains in the positive-fidelity regime across runs.
Right: a complementary E5 tradeoff view, showing that BGSS and the risk-only
same-budget controller attain similar mean onset-event counts, while BGSS
improves over the random same-budget controller on both onset events and mean
final mismatch.}
\label{fig:app:e35_pair}
\end{figure}

\end{document}